\documentclass[twoside]{article} %
\usepackage{Gan,times}

\usepackage{hyperref}
\usepackage{url}

\usepackage{amsmath,amsfonts,bm}

\def\eqref#1{equation~\ref{#1}}

\def\1{\bm{1}}

\DeclareMathAlphabet{\mathsfit}{\encodingdefault}{\sfdefault}{m}{sl}
\SetMathAlphabet{\mathsfit}{bold}{\encodingdefault}{\sfdefault}{bx}{n}

\usepackage{mathtools}
\mathtoolsset{showonlyrefs=true}

\usepackage{amsthm}
\usepackage{float}
\usepackage[normalem]{ulem}

\usepackage{siunitx}
\theoremstyle{plain}
\newtheorem{theorem}{Theorem}[section]
\newtheorem{proposition}[theorem]{Proposition}

\theoremstyle{definition}
\newtheorem{definition}[theorem]{Definition}

\theoremstyle{remark}
\newtheorem{remark}[theorem]{Remark}

\newcommand{\gbar}{\mathsf{g}}

\usepackage{graphicx}
\graphicspath{{figs/}}
\usepackage{caption}
\usepackage{subcaption}
\usepackage{multirow}

\makeatletter
\def\hlinewd#1{%
\noalign{\ifnum0=`}\fi\hrule \@height #1 \futurelet
\reserved@a\@xhline}
\makeatother

\usepackage[linesnumbered,ruled,vlined]{algorithm2e} %
\SetKwInput{KwRequire}{Require}
\SetKwInput{KwFunction}{Notation}
\SetKwInput{KwChoice}{Choice}

\title{Graph-Informed Adversarial Modeling: \\Infimal Subadditivity of Interpolative\\ Divergences}
\shorttitle{Graph-Informed Adversarial Modeling}

\iclrfinalcopy

\author{Panagiota Birmpa\\
School of Mathematical and Computer Sciences\\
Heriot--Watt University\\
Edinburgh EH14 4AS, UK\medskip\\
Maxwell Institute for Mathematical Sciences\\
Edinburgh EH8 9BT \medskip\\
\texttt{P.Birmpa@hw.ac.uk} \\
\And
Eric Joseph Hall\\
School of Mathematical and Computer Sciences\\
Heriot--Watt University\\
Edinburgh EH14 4AS, UK\medskip\\
Maxwell Institute for Mathematical Sciences\\
Edinburgh EH8 9BT \medskip\\
\texttt{e.hall@hw.ac.uk} \\
}
\shortauthors{Birmpa and Hall}

\begin{document}

\maketitle
\thispagestyle{firstpage}

\begin{abstract}
We study adversarial learning when the target distribution factorizes according to a known Bayesian network. For interpolative divergences, including $(f,\Gamma)$-divergences, we prove a new infimal subadditivity principle showing that, under suitable conditions, a global variational discrepancy is controlled by an average of family-level discrepancies aligned with the graph. In an additive regime, the surrogate is exact. This closes a theoretical gap in the literature; existing subadditivity results justify graph-informed adversarial learning for classical discrepancies, but not for interpolative divergences, where the usual factorization argument breaks down. In turn, we provide a justification for replacing a standard, graph-agnostic GAN with a monolithic discriminator by a graph-informed~GAN~(GiGAN) with localized family-level discriminators, without requiring the optimizer itself to factorize according to the graph. We also obtain parallel results for integral probability metrics and proximal optimal transport divergences, identify natural discriminator classes for which the theory applies, and present experiments showing improved stability and structural recovery relative to graph-agnostic baselines. 
\end{abstract}

\noindent\textbf{Keywords:} generative adversarial networks, Bayesian networks, variational learning, $(f,\Gamma)$-divergences, infimal subadditivity, structured generative models, multi-discriminator learning

\section{Introduction}
\label{sec:intro}

A Generative Adversarial Network (GAN) \citep{goodfellow2014generative} is an artificial neural network architecture used in learning a probability distribution from a given dataset by generating new data samples that resemble a given dataset.  Recent advances in GANs have been driven by a deeper understanding of the underlying probability divergences and distances that govern their training. Frameworks such as $f$-GANs \citep{nowozin2016f} generalize the objective to any $f$-divergence, enabling flexible modeling through variational estimation, while Wasserstein GANs \citep{arjovsky2017wasserstein, gulrajani2017improved} reformulate the problem using the Earth Mover’s distance, offering improved gradient behavior and stability via the integral probability metric framework. More recent developments use divergences which interpolate between $1$-Wasserstein and Kullback--Leibler (KL) divergence as developed in \cite{dupuis2022formulation}. This divergence has been extended to $(f,\Gamma)$-divergences \citep{birrell2022f}, a broader family that interpolates between integral probability metrics (IPMs) \citep{muller1997integral} and $f$-divergences, and has been extensively examined for its role in the design and theoretical analysis of generative models, particularly with respect to stability and robustness \citep{dupuis2022formulation,birrell2022f,chen2024robust,gu2024lipschitz}. Many other divergences have been developed, for example, see \cite{stein2024wasserstein, terjek2021moreau, baptista2025proximal, birrell2023optimal, birrell2022function}.

GANs are often trained without explicitly exploiting known conditional dependence structure in the target distribution. In this paper, we refer to such approaches as \emph{graph-agnostic} when the discriminator acts on the full joint distribution and does not use a known graphical model constructed from side information. In some applications, however, side information about the distribution is available, such as known physical/chemical laws, covariates, design constraints, and expert judgments. This perspective fits within a broader line of work in which probabilistic graphical structure is used to inform neural network learning architectures and surrogate models, see, e.g. \cite{HallTaverniersEtAl:2020aa}. In \cite{feizi2017understanding} the authors proposed GAN architectures tailored to multivariate Gaussian distributions, while \cite{balaji2019normalized} and \cite{farnia2023gat}  extended the analysis to GANs designed for mixtures of Gaussians. \cite{birrell2022structure} introduces structure-preserving GANs which leverage additional structure, such as group symmetry, to learn data-efficiently via new variational representations of divergences. The approach reduces the discriminator space to its projection onto the invariant subspace using conditional expectation with respect to the $\sigma$-algebra defined by the structure. 

However, in this work we shift our focus to cases where conditional dependencies of variables involved in the model could be represented by a Bayesian network (see \cite{pearl1988probabilistic,koller2009probabilistic}), e.g., the variables might follow a Markov chain. Formally, a Bayesian network is a joint probability distribution on a collection of random variables indexed by the vertices of a Directed Acyclic Graph  (DAG). Let $G = (V,E)$ be a DAG with vertex set $V = \{1, \dots, n\}$, $n\in\mathbb{N}$, and edge set $E \subseteq V \times V$. We associate to each vertex $i \in V$ a random variable $X_i$ taking values in a space $\mathcal{X}_i$, and write $\bm{X} = \bm{X}_V = (X_1,\dots,X_n)$ with state space $\mathcal{X} = \mathcal{X}_1 \times \dots \times \mathcal{X}_n$. The Bayesian network is then the pair $(G,Q)$, where the joint distribution $Q$ of $\bm{X}$ factorizes according to $G$ as 
\begin{equation*}
  Q(dx_1,\dots, dx_n)
= \prod_{i=1}^{n}
Q_{i|{\mathrm{Pa}}(i)}(
  dx_i)\,,
\end{equation*}
where $\mathrm{Pa}(i)$ denotes the parents of vertex $i$ and $Q_{i|\mathrm{Pa}(i)}=Q(dx_i|\bm x_{\mathrm{Pa}(i)})$ are the corresponding conditional probability distributions (CPDs) or marginal distributions if $\mathrm{Pa}(i) = \emptyset$. Intuitively, a Bayesian network encodes which variables directly depend on which others, providing a structured factorization of the joint distribution. In this paper, we refer to GANs that explicitly use this graphical information in the design of the discriminator objective as \emph{graph-informed}.

Theoretical foundations for exploiting a known Bayesian network or undirected Markov random field structure to design graph-informed GANs with multiple simple discriminators, each enforcing constraints on local neighborhoods of the Bayesian network or Markov random field, were established in \cite{ding2021gans}, building on prior results in \cite{daskalakis2017square}. This framework relies on the subadditivity properties of classical measures of discrepancy between probability distributions---such as KL divergence, Jensen--Shannon (JS) divergence, and Wasserstein distance---in the context of Bayesian networks and Markov random fields with or without mild conditions. These properties provide upper bounds on the divergence between two high-dimensional distributions sharing the same graphical structure by expressing it as a sum of divergences between their conditionals/marginals over local neighborhoods. Here and throughout the paper, we use \emph{classical divergences} to refer to unrestricted discrepancy functionals such as $f$-divergences, and \emph{interpolative divergences} to refer to discrepancies such as $(f,\Gamma)$-divergences, which constrain the discriminator's function class.

In generative modeling, a monolithic discriminator acting on a high-dimensional, joint distribution poses both computational and statistical challenges: the associated min--max optimization is costly, and discrimination via hypothesis testing in high dimensions can become increasingly sample-inefficient. Motivated by these issues, we establish an analogous variant of the subadditivity property on Bayesian networks, referred to as \emph{$\alpha$-linear generalized infimal subadditivity}. More precisely, for a measure of discrepancy $\delta$ with respect to measures of discrepancy $\delta_1, \dots, \delta_n$, this new property takes the form,
$$ \inf_{P \in \mathcal{P}(\mathcal{X})} \delta(Q,P)
  \leq
  \inf_{P \in \mathcal{P}(\mathcal{X})}\alpha
\sum_{i=1}^n\delta_i\left(Q_{i\,\cup\,\mathrm{Pa}(i)}\,, P_{i\,\cup\,\mathrm{Pa}(i)}\right),$$
for all Bayesian networks $Q$, where $Q_{i\,\cup\,\mathrm{Pa}(i)}$, $P_{i\,\cup\,\mathrm{Pa}(i)}$ are the marginals of $Q$ and $P$ over the $i$-th node and its parents. We show that $(f,\Gamma)$-divergences and other interpolative divergences satisfy $\alpha$-linear generalized infimal subadditivity, whereas classical $f$-divergences satisfy $\frac{1}{n}$-linear generalized infimal additivity, that is, the case where the above inequality holds with equality.

This extends the scope of graph-informed adversarial learning that can be justified. Existing results explain localized discriminators mainly through classical subadditivity properties of unrestricted discrepancies; by introducing infimal subadditivity, we show that an analogous justification continues to hold for interpolative divergences, including $(f,\Gamma)$-divergences, where the classical argument breaks down. In generative modeling, this means that when the target has a known Bayesian-network structure, a graph-informed GAN trained under an interpolative divergence objective (i.e.,~GiGAN) can replace a graph-agnostic GAN with a monolithic discriminator through a principled surrogate.

This paper is organized as follows. In the next section, we place our contribution in the broader literature on structure-aware generative modeling and adversarial learning. In Section~\ref{sec:var-gans}, we recall the variational approach for learning adversarial generative models and key statistical discrepancies. We introduce localized function classes and conditional expectation operators used to encode the factorization of the target probability distribution with respect to a Bayesian network in Section~\ref{sec:bn-operators}. Our main theoretical results, namely Theorems~\ref{thm:main2general}, \ref{thm:main_IPM}, and \ref{thm:pOTd}, which prove infimal subadditivity for $(f,\Gamma)$-divergences, IPMs, and optimal transport divergences, are presented in Section~\ref{sec:main-results}. These provide guarantees (certificates) for replacing global discrepancy minimization by a principled, graph-informed surrogate. Finally, we present several numerical case studies in Section~\ref{sec:experiments} and discuss implications and future work in Section~\ref{sec:discussion}.

\section{Related work}
\label{sec:related-work}

Existing work on incorporating structure into adversarial generative modeling can be broadly divided into two strands. The first builds structure directly into the generator architecture or training procedure, for example, through causal mechanisms, fairness constraints, or conditional independence objectives. The second uses multiple discriminators or localized adversarial objectives as a practical strategy for high-dimensional generation, typically without a general variational justification. Our work is closest in spirit to the latter, but differs in that it provides theory and foundations, namely a general graph-informed adversarial learning strategy for interpolative divergences on Bayesian networks, continuing in the spirit of \cite{ding2021gans}.

Within the first strand, \cite{kocaoglu2017causalgan} were among the first to integrate structural causal models into GANs; in this work, the generator is explicitly designed to reflect an assumed causal graph, making it a pioneering framework for learning causal implicit generative models through adversarial training. \cite{moraffah2020causal} propose CAN, a two-stage GAN framework that learns causal structure from data and can generate both conditional and interventional samples by embedding a structural causal model and an intervention mechanism into the adversarial architecture. \cite{wen2022causal} proposes Causal-TGAN for tabular data generation, incorporating predefined causal relationships into GANs to enhance data fidelity, while \cite{feng2022principled} also studies the incorporation of structural causal elements into generator design. In the fairness literature, \cite{xu2019achieving} introduces CFGAN, which enforces causal fairness criteria such as total effect and direct/indirect discrimination, and \cite{van2021decaf} develops DECAF, a GAN that embeds structural causal models into the generator's input layers to produce fair synthetic data while capturing the underlying causal mechanisms critical for fairness. Related, \cite{ahuja2021conditionally} proposes a GAN-based framework that generates data distributions close to the original while enforcing conditional independence constraints through an additional adversarial objective, and which directly ensures fairness, and \cite{casajus2022evolutive} proposes a Bayesian-network autoencoder trained adversarially for interpretable anomaly detection.

These prior studies focus primarily on incorporating structure through the generator, through auxiliary constraints, or through the overall training architecture. That is, they focus on injecting structure through model design. By contrast, our contribution is on the discriminator-side and variational. %
We study how known Bayesian network structure can be used to replace a graph-agnostic monolithic discriminator objective by graph-informed family-level objectives, and we show that this yields a principled surrogate for the corresponding graph-constrained global discrepancy. Methodologically, this connects to empirical multicritic strategies used, for example, in image-to-image translation methods \citep{yu2019free}, recurrent neural networks and time-series generation \citep{hyland2018real}, convolutional neural networks \citep{nie2018relgan}, and inpainting \citep{chang2019free}. However, our results supply the missing theoretical rationale for these largely heuristic approaches: when the target factorizes according to a Bayesian network, sums of local graph-informed discrepancies control the corresponding global objective.

A separate but related strand relates to graph-structure generation, where graph-structured and conditional generative models have been explored in \cite{you2018graphrnn}, incorporating relational structure into generation. In the broader context of graph neural networks (GNNs), originally introduced in
\cite{scarselli2008graph} with earlier foundational ideas appearing in
\cite{sperduti1997supervised}, our work is conceptually related in that both
frameworks exploit graph-induced locality to replace a global high-dimensional
learning problem by a collection of structured local operations. In GNNs, this is achieved through message passing and feature aggregation over local
neighborhoods. Such architectures have been particularly successful in material science, where graph-based representations naturally encode atomic and molecular
interactions; see, for example, \cite{xie2018crystal,reiser2022graph,du2024densegnn}. Furthermore, GNNs equivariant to rotations have been studied in \cite{satorras2021n}. By contrast, our approach uses local variational discriminators defined on each node and its parent set in order to control a divergence between full probability models.

Our contribution is not to introduce another architecture-specific mechanism for injecting structure, but to develop a general variational principle for graph-informed adversarial learning. In contrast to prior work, which relied on classical subadditivity properties of unrestricted divergences, we show that interpolative divergences admit a weaker but still useful infimal subadditivity principle. This fills a theoretical gap in the literature and provides a general justification for graph-informed localized discriminators in variational generative modeling. To make this precise, we next review the variational formulation of GAN training and the discrepancy functionals that underpin our later subadditivity and infimal subadditivity results.

\section{Variational learning in GANs}
\label{sec:var-gans}

The purpose of this section is to set out the variational framework and discrepancy functionals used in the remainder of the paper. In particular, we review $f$-divergences, IPMs, $(f,\Gamma)$-divergences, and proximal optimal transport divergences within a common language that will support the later graph-informed analysis.

GANs are a class of neural architectures for learning probability distributions from data by training a generator to synthesize samples that mimic a target dataset. Training is formulated as a zero-sum game between a generator and a discriminator that compete against each other. Formally, let $(\mathcal{X},\mathcal{M})$ be a measurable space and $\mathcal{P}(\mathcal{X})$ be the space of probability distributions on $\mathcal{X}$. For a fixed target distribution $Q\in \mathcal{P}(\mathcal{X})$, GAN learning is formulated as a min--max problem for some objective functional $H:\Gamma\times\mathcal{P}(\mathcal{X})\times \mathcal{P}(\mathcal{X})\to \bar{\mathbb{R}}$, i.e., 
\begin{equation}\label{eq:general GAN formulation}
\inf_{g\in \mathcal{G}}\sup_{\gamma\in\Gamma}H(\gamma;Q,P_g)\,,
\end{equation}
where $\mathcal{G}$ and $\Gamma$ are the spaces of generators and discriminators, respectively, and are parametrized by neural networks. The generator $g:Z\to \mathcal{X} $ is a sufficiently expressive mapping that transforms a random noise vector $z\in Z$, sampled from a prior distribution $P_Z$ (e.g., Gaussian or uniform), into a synthetic data point $g(z)\in \mathcal{X}$ in the same space as the real data. The mapping $g$ pushes the prior distribution forward to the generated probability distribution $P_g =P\circ g^{-1}$, with the aim that $P_g$ is indistinguishable from the real data, $Q$. On the other hand, the discriminator $\gamma: \mathcal{X}\to \mathbb{R}$ then attempts to distinguish samples from $Q$ and $P_g$. 

\subsection{Training via variational representations}
\label{sec:variational-representations}

We begin by explaining how the min--max problem in \eqref{eq:general GAN formulation} can be interpreted as minimizing a variationally defined discrepancy between probability measures. In most GANs, \eqref{eq:general GAN formulation} corresponds to minimizing a measure of discrepancy between the target probability distribution $Q$ and the generated $P_g$. More generally, for any $P, Q\in \mathcal{P}(\mathcal{X})$, under suitable choices of the objective functional $H$ and the function class $\Gamma$, the quantity, 
$$\mathcal{D} (Q, P; H, \Gamma) = \sup_{\gamma\in\Gamma} H(\gamma;Q,P)\,,$$ 
 can be a divergence, an Integral Probabilistic Metric (IPM), or an optimal transport cost between probability distributions $Q$ and $P$. In this notation,  \eqref{eq:general GAN formulation} has the representation, 
\begin{equation}\label{eq: GAN measures of discrepancy}
\inf_{g\in \mathcal{G}}\sup_{\gamma\in\Gamma}H(\gamma;Q,P_g)=\inf_{g\in \mathcal{G}} \mathcal{D} (Q, P_g; H, \Gamma) \,.
\end{equation}
However, not every measure of discrepancy possesses the
\emph{divergence property}, namely,
\begin{equation*}
  \mathcal{D}(Q, P; H, \Gamma)=0 \quad \text{if and only if} \quad Q=P \,.
\end{equation*}
This property is important from a statistical and optimization perspective because it ensures that achieving zero discrepancy is equivalent to recovering the target distribution itself. In \cite{birrell2022f}, the authors identify sufficient conditions under which this property holds for several examples of such discrepancies, including those considered in this work. %
For the convenience of the reader, relevant definitions of admissible sets are provided in Appendix~\ref{appendix:extra-definitions}.

The next subsections instantiate this general template for the main discrepancy classes relevant to GAN training and to our later analysis.

\subsection{\texorpdfstring{$f$-Divergences}{f-Divergences}}
\label{sec:f-div}

We begin with $f$-divergences, since they provide one of the classical variational formulations used in GAN training and motivate the role of convex conjugates in discriminator objectives. Let  $f:[0,\infty)\to\mathbb{R}$ be strictly convex and lower semicontinuous with $f(1)=0$. For probability distributions $P, Q \in \mathcal{P}(\mathcal{X})$, the $f$-divergence of $Q$ with respect to $P$ is defined by 
\begin{equation}\label{def:f_divergence}
\mathfrak{D}_f(Q\|P) = E_P\left[f\left(\frac{dQ}{dP}\right)\right]\,,
\end{equation}
if $Q\ll P$, and set to $+\infty$ otherwise, where $dQ/dP$ denotes the Radon--Nikodym derivative. This divergence admits the variational representation,
\begin{equation}\label{eq:VarfDiv}
  \mathfrak{D}_f(Q\|P)=  \sup_{\gamma\in \mathcal{M}_b(\mathcal{X})}\left\{E_Q[\gamma]-\inf_{\nu \in \mathbb{R}}\left\{ \nu + E_P[f^*(\gamma -\nu)]\right\}  \right\} \,,
\end{equation}
where $f^*(s)=\sup_{t\in [0,\infty)}\left\{ st- f(t)\right\}$ is the Legendre--Fenchel transform (convex conjugate) of $f$. For variational training of GANs, the choice of $f$ determines both the divergence $\mathfrak{D}_f(Q\|P)$ and the nonlinearity in the discriminator through its convex conjugate $f^*$. In particular, the growth of $f^*$ controls how strongly large values $\gamma(x)$ are penalized under the reference distribution $P$. 

In the sequel, we will use the Donsker--Varadhan variational formula for KL,
\begin{equation}\label{eq:formula_DV}
\mathfrak{D}_{\rm DV}(Q\|P)=  \sup_{\gamma\in \mathcal{M}_b(\mathcal{X})}\left\{E_Q[\gamma]-\log E_P[e^\gamma]  \right\} \,,
\end{equation}
and, when considering $f$-GAN type losses, we will also use \eqref{eq:VarfDiv} with both $f_{\mathrm{KL}}(t)=t\log t$, for the KL divergence, and 
$$f_{\mathrm{JS}}(t) = \tfrac{1}{2} \left[t \log t - (t+1) \log\left(\frac{t+1}{2}\right)\right]\,, \qquad t \geq 0\,$$
for the JS divergence.
The corresponding Legendre--Fenchel transforms are $f^*_{\mathrm{KL}}(s)=e^{s-1}$ and $f^*_{\mathrm{JS}}(s) = - \log(2-e^s)$, $s<\log2$  (see, e.g., \cite[Tables 1 and 2]{nowozin2016f}). 
In the KL case, $f^*_{\mathrm{KL}}$ has exponential growth, which heavily penalizes large discriminator values and emphasizes rare events where $P$ and $Q$ disagree in the tails. Moreover, for the KL divergence the infimum over $\nu$ in \eqref{eq:VarfDiv} can be solved analytically and one recovers the classical Donsker--Varadhan representation \eqref{eq:formula_DV}, in which the discriminator objective involves a log-moment generating function, a form that underlies many existing $f$-GAN and risk-sensitive objectives. Unlike KL, the JS divergence is symmetric and bounded, $0 \leq \mathfrak{D}_{\mathrm{JS}}(Q\|P) \leq \log 2$, and the corresponding conjugate $f^*_{\mathrm{JS}}$ diverges only logarithmically as $s \to \log 2$, so the penalty on extreme discriminator values is much milder than when using KL. In the JS case, the variational representation recovers the original logistic GAN objective (up to an additive constant and a reparameterization of $\gamma$). In our experiments involving discrete Bayesian networks, we therefore focus on the JS divergence, since its symmetry and boundedness (it remains finite even when $P$ and $Q$ have partially mismatched support) tend to yield more stable training than KL-based objectives in this setting.

\subsection{\texorpdfstring{$\Gamma$-Integral Probability Metrics ($\Gamma$-IPMs)}{Gamma-Integral Probability Metrics (Gamma-IPMs)}}
\label{sec:Gamma-IPMs}

We next turn to $\Gamma$-IPMs, which arise when the discriminator is constrained to a prescribed function class and include several widely used probability metrics. Let $\Gamma \subset \mathcal{M}_b(\mathcal{X})$ be a class of bounded measurable test functions on $\mathcal{X}$, and define the associated $\Gamma$-IPM by,
\begin{equation}\label{eq:IPM}
  W^{\Gamma}(Q,P)=  \sup_{\gamma\in \Gamma}\left\{E_Q[\gamma]-E_P[\gamma]\right\} \,.
\end{equation}
Under  certain conditions for $\Gamma\subset \mathcal{M}_b(\mathcal{X})$ \citep[Theorem 8]{birrell2022f}, $W^{\Gamma}$ satisfies the divergence property. If $\Gamma\subset C_b(\mathcal X)$ is strictly admissible (see Appendix~\ref{appendix:extra-definitions}), then $W^{\Gamma}$ also satisfies the divergence property \citep[Theorem 15]{birrell2022f}. Different choices of $\Gamma$ yield many classical probability metrics, including Wasserstein distances, total variation, the Dudley metric, and the maximum mean discrepancy. In the context of GANs, choosing $\Gamma$ corresponds to specifying the discriminator function class, and hence determines which IPM is minimized during training. 

More generally, taking $\Gamma = \{\gamma : \mathcal{X} \to \mathbb{R}\}$ to be:
\begin{enumerate}
\item $\Gamma = {\rm{Lip}}_b^1(\mathcal{X})$ yields the $1$-Wasserstein distance and its variants,
\item $\Gamma = \{\gamma\in C_b(\mathcal{X}): \|\gamma\|_{\infty} \leq 1\}$ yields (a multiple of) the total variation distance,
\item $\Gamma = \{\gamma\in {\rm{Lip}}_b^1(\mathcal{X}): \|\gamma\|_{\infty} \leq 1\}$ yields the Dudley metric, and
\item $\Gamma = \{\gamma\in \mathcal{H} : \|\gamma\|_{\mathcal{H}}\leq1\}$, i.e., the unit ball in a RKHS $\mathcal{H} \subset C_b(\mathcal{X})$, yields a generalization of maximum mean discrepancy.
\end{enumerate}
These examples illustrate how regularity and boundedness constraints on $\Gamma$ translate into different notions of ``distance'' between $Q$ and $P$, and hence into different inductive biases in the associated GAN discriminators, determining which discrepancies between $Q$ and $P$ are emphasized during training.

\subsection{\texorpdfstring{$(f,\Gamma)$-Divergences}{(f, Gamma)-Divergences}}
\label{sec:f-Gamma-div}

The viewpoints from Sections~\ref{sec:f-div} and \ref{sec:Gamma-IPMs} can be combined in the form of $(f,\Gamma)$-divergences, which interpolate between $f$-divergences and IPMs and will be central to our later results. Recall from \eqref{eq:VarfDiv} that the $f$-divergence can be written in variational form as a supremum over all bounded measurable functions. In practice, however, both in statistical estimation and in GAN training, the discriminator is restricted to a smaller function class \citep{birrell2022f}. Given a convex $f$ and a class of test functions $\Gamma \subset \mathcal{M}_b(\mathcal{X})$, we define the corresponding $(f,\Gamma)$-divergence by,
\begin{equation}
    \label{eq:f-gamma-variational-rep}
    \mathfrak{D}_{f}^{\Gamma}(Q\|P) =\sup_{\gamma\in\Gamma}\left\{E_{Q}[\gamma]-\inf_{\nu \in \mathbb{R}}\left\{ \nu + E_{P}[f^*(\gamma -\nu)]\right\}\right\}\,,
\end{equation}
where $f^*$ is the convex conjugate of $f$. When $\Gamma = \mathcal{M}_b(\mathcal{X})$, then \eqref{eq:f-gamma-variational-rep} reduces to the $f$-divergence $D_f(Q\|P)$. In \cite[Theorem 8]{birrell2022f}, the authors prove that under suitable conditions, $\mathfrak{D}_{f}^{\Gamma}(Q\|P)$ is a divergence, i.e. $\mathfrak{D}_{f}^{\Gamma}(Q\|P)\ge 0$ with equality if and only if $Q=P$, and interpolates between $f$-divergences and $\Gamma$-IPMs via an infimal convolutional formula, i.e., 
\begin{equation}\label{def:inf con}
\mathfrak{D}_{f}^{\Gamma}(Q\|P) = 
  \inf_{R\in\mathcal{P}(\mathcal{X})}
    \left\{W^{\Gamma}(Q,R) + \mathfrak{D}_f(R\|P) \right\}\,.
\end{equation}
In particular, we have ${\displaystyle 0\leq \mathfrak{D}_{f}^{\Gamma}(Q\|P) \leq\min\left\{ \mathfrak{D}_f(Q\|P),\; W^{\Gamma}(Q,P)\right\}}$. In  \cite[Theorem 15]{birrell2022f}, the same result is proven when $\Gamma \subset C_b(\mathcal X)$ and $\mathcal X$ is a Polish
space. In Theorem 25 of the same paper, it is proven that there exists a unique optimizer $R^{*} \in\mathcal{P}(\mathcal{X})$ in \eqref{def:inf con} and there exists an optimizer $\gamma^{*}\in \Gamma$ in \eqref{eq:f-gamma-variational-rep} which is unique (up to a constant) in $\textrm{supp}(P)\cap \textrm{supp}(Q)$. Moreover $\gamma^{*}$ and $R^{*}$ satisfy 
\begin{equation}\label{eq:opt f, gamma}
dR^{*}\propto (f^*)'(\gamma^{*})dP\,, \quad 
  \text{or equivalently}  \quad  
    \gamma^{*}  =f'\left(\frac{dR^{*}}{dP}\right) + \text{constant}\,,
\end{equation}
and we have $ W^{\Gamma}(Q,R^{*}) = E_Q[\gamma^{*}] - E_{R^{*}}[\gamma^{*}]\,.$ 

We introduce the set $\Gamma^L=\{L\gamma:\gamma\in\Gamma\}$ for $L>0$. For the $(f,\Gamma^L)$-divergence, the optimizers of \eqref{eq:f-gamma-variational-rep} and \eqref{def:inf con} depend on $L$ and consequently change with different choices of $L$.

In the GAN setting, $(f,\Gamma)$-divergences capture exactly the discrepancy that can be detected by a discriminator constrained to the function class $\Gamma$ (e.g., a neural network architecture with Lipschitz or locality constraints), and will be the central objects in our subadditivity results for Bayesian network-aligned discriminators.

\subsection{Proximal optimal transport divergences}

Finally, we consider proximal optimal transport divergences, which provide an interpolation between transport-based and information-theoretic discrepancies and fit naturally into the variational framework developed above. Proximal optimal transport divergence, introduced in \cite{baptista2025proximal}, is a discrepancy measure that interpolates between information divergences $\mathfrak{D}$ and optimal transport distances $T_c$ corresponding to a cost function $c:\mathcal{X}\times \mathcal{X}\to[0,\infty)$, via an infimal convolution formulation,
\begin{equation}
  \label{def:POTD}
  \mathfrak{D}^c_\varepsilon(Q\|P) = \inf_{R}\left\{ T_c(Q,R) + \varepsilon \mathfrak{D}(R\|P) \right\}\,,
\end{equation}
where the infimum is over all probability distributions $R$ on a Polish space $\mathcal{X}$. As for the $(f,\Gamma)$-divergence, the dual formulation offers a variational representation involving test functions. For example, for  $f=f_{{\rm{KL}}}$ and $c$ bounded below and lower semicontinuous, the dual representation is given by
\begin{equation}\label{eq:dual for intro}
\mathfrak{D}_{\mathrm{KL},\varepsilon}^{c}(Q\|P)= 
\sup_{\substack{(\phi, \psi) \in C_b(\mathcal{X})\times C_b(\mathcal{X}),\\ \phi\oplus\psi\leq c}}
\left\{ E_Q[\phi]  -\varepsilon \log E_P[e^{-\psi/\varepsilon}]\right\},
\end{equation}
where $\phi\oplus\psi$ denotes the function $\phi(x)+\psi(y)$. While the dual problem is written over bounded continuous functions for convenience, optimal potentials need not be bounded in general; however, there exist $\phi \in L^1(Q)$ and $\psi \in L^1(P)$ satisfying $\phi\oplus\psi \le c$. 

A special class of optimal transport distances on a metric space $(\mathcal X,d)$ is when $T_c$ are $p$-Wasserstein distances, 
\begin{equation}\label{def:wass.p}
W_p(Q, P) = \inf_{\pi \in \Pi(Q, P)} \left( \int_{\mathcal X \times \mathcal X} d(x,y)^p \, d\pi(x, y) \right)^{\frac{1}{p}},
\end{equation}
where $c(x,y) = d(x,y)^p$ for $p\geq 1$, $Q,P\in \mathcal{P}_p(\mathcal{X}) = \left\{ P \in \mathcal{P}(\mathcal X) \,:\,  \int d(x,x_0) ^p dP < \infty\right\}$, and $x_0$ is an arbitrary point in $\mathcal{X}$. When $T_c$ is the $1$-Wasserstein distance, then \eqref{def:POTD} becomes an $(f,\Gamma)$-divergence with $f=f_{{\rm{KL}}}$ and $\Gamma={\rm{Lip}}^1_{b}(\mathcal{X})$.

Our aim is to exploit the graphical structure arising from side information. When $Q$ is a target Bayesian network, the discrepancies introduced above do not yet exploit this structural information. In the remainder of the paper, we specialize to distributions that factorize according to Bayesian networks, as described in the next section. In Section \ref{sec:main-results}, we then establish a $\frac1n$-linear infimal subadditivity property that will allow us to define localized discriminator classes and to relate global objective functionals to lower-dimensional components associated with the graph.

\section{Bayesian network operators}
\label{sec:bn-operators}

We now introduce the Bayesian network framework that provides the structural setting for the variational discrepancies considered in the previous section. Our aim is to encode the factorization of the target probability distribution with respect to a DAG and to define the notation needed for localized function classes and conditional expectation operators. These objects will be used in the sequel to formulate and analyze graph-aligned discriminator objectives.

Continuing the Bayesian network notation from Section~\ref{sec:intro}, for a DAG denoted $G = (V,E)$, the random vector $\bm{X} =  \bm{X}_V = (X_1,\dots,X_n)$ indexed by the vertices $V$ takes values $\bm x=(x_1,\dots,x_n)$ in the state space $\mathcal{X} = \mathcal{X}_1 \times \dots \times \mathcal{X}_n$. We denote the joint probability distribution of $\bm X$ by $Q$ (with density $q$ if it exists). For any subset $A=\{i_1,\dots,i_m\}\subset V$, we denote $\bm X_A=(X_{i_1},\dots,X_{i_m})$ which takes values $\bm X_A=\bm x_A=(x_{i_1},\dots,x_{i_m})\in\mathcal X_A=\mathcal X_{i_1}\times \cdots \times X_{i_m}$ and its marginal is denoted by $Q_A$.  The conditional probability distribution of $X_i$ given parents values $\bm{X}_{\mathrm{Pa}(i)} = \bm{x}_{\mathrm{Pa}(i)}$ is denoted by $Q_{i \mid\mathrm{Pa}(i)}$, i.e, $Q_{i \mid\mathrm{Pa}(i)}(dx_i)=Q(dx_i \mid \bm{x}_{\mathrm{Pa}(i)})$ with corresponding density $q(x_i |\bm{x}_{\mathrm{Pa}(i)})$. In general, for two sets of indices $A,B\subset V$, $Q_{A \mid B}(d\bm x_A)=Q(d\bm x_A |  \bm{x}_{B})$. For notational simplicity, when the index sets are clear, we write $Q_{i_1\cdots i_m}$ in place of $Q_{\{i_1,\dots,i_m\}}$ and 
$Q_{i_1\cdots i_m \mid j_1\cdots j_k}$ in place of
$Q_{\{i_1,\dots,i_m\}\mid \{j_1,\dots,j_k\}}$, e.g., in the proof of the main theorem. Having all notation and terminology in place, we next define Bayesian networks formally.

\begin{definition}[Bayesian network over a DAG]\label{def:BN_over_DAG} 
  A Bayesian network over a DAG $G = (V,E)$ with variables $X_i \in \mathcal{X}_i$ is the pair $(G,Q)$ such that the joint probability distribution $Q$ factorizes as 
  \begin{equation}\label{eq:factorization}
  Q(dx_1,\dots, dx_n)
  = \prod_{i=1}^{n}
  Q_{i \mid\mathrm{Pa}(i)}(
    dx_i) \,,
  \end{equation}
  with $\mathrm{Pa}(i)$ the parent set (and $Q_{i \mid\mathrm{Pa}(i)}$ the CPDs, or marginals if $\mathrm{Pa}(i)=\emptyset$).
\end{definition}
For any set $A\subset V$, we denote $A^c=V\setminus A$. The set of probability distributions that factorize according to $G$ with vertex set $V=\{1,\dots,n\}$ and edge set $E\subset V\times V$ is denoted by $\mathcal{P}^G$. Given a Bayesian network $(G,Q)$, the marginal
distribution $Q_A$ of $\bm X_A$ has the form 
\begin{equation*}
 Q_A(d\bm x_A)=\int_{\mathcal X_{\rho_A}}\prod_{i\in A}Q_{i \mid\mathrm{Pa}(i)}(dx_i)\prod_{j\in \rho_A}Q_{j \mid\mathrm{Pa}(j)}(dx_j) 
\end{equation*}
where $\rho_A$ is the set of indices of all the ancestors of $A$ with respect to graph $G$. 

We now fix $Q\in\mathcal{P}^G$. For each vertex $i\in V$, we the following spaces of probability distributions: 
\begin{equation}\label{def:general_PrMeasQ}
\begin{split}
\mathcal{P}_i^G&=\Bigg\{Q_{V^{(i-1)} \mid\mathrm{Pa}(i)}\otimes  P_{i\,\cup\,\mathrm{Pa}(i)}\otimes \prod_{j=i+1}^n Q_{j \mid\mathrm{Pa}(j)}\,: P_{i\,\cup\,\mathrm{Pa}(i)}\in \mathcal{P}\left(\mathcal{X}_{i\,\cup\,\mathrm{Pa}(i) }\right)\Bigg\}
 \end{split}
\end{equation}
where $V^{(i)}=\{1,\dots, i\}\,\setminus\,\mathrm{Pa}(i+1)$ with $i=1,\dots,n-1$, and the symbol $\otimes$ between probability distributions means the product of distributions, i.e., if $\mu_1,\dots,\mu_n$ are probability distributions on $\mathcal X_1,\dots,\mathcal X_n$, respectively, then 
$$\mu_1\otimes\cdots\otimes\mu_n(d\bm x)=\mu_1(dx_1)\cdots \mu_n(dx_n)\,.$$
The probability space given in \eqref{def:general_PrMeasQ} maintains the graph-structure of $G$. For  $\gamma \in \Gamma$, we introduce the conditional $Q$-expectation operators as follows:
\begin{equation}\label{eq:general_I_operator}
 I_{i,{\rm{Pa}}(i)}[\gamma](x_i,{\bm x}_{\mathrm{Pa}(i)}) \!=\! \! \int_{{\bm{\mathcal{X}}}_{V^{(i-1)}}}\!\!\int_{\mathcal{X}_{i+1}} \!\!\!\!\!\cdots \! \int_{\mathcal{X}_{n}} \!\!\!\!\gamma({\bm x})\!\! \prod_{j=i+1}^n \!\! Q_{j \mid\mathrm{Pa}(j)}(dx_j)\,Q_{V^{(i-1)}\mid  \mathrm{Pa}(i) }(d{\bm x}_{V^{(i-1)}})\,.
 \end{equation}
Moreover, for each vertex $i\in V$, we define the following sets:
\begin{equation}\label{eq:general_projectedToXY}
I_{i,\mathrm{Pa}(i)}[\Gamma]=\left\{I_{i,\mathrm{Pa}(i)}[\gamma]:\gamma\in\Gamma\right\}
\end{equation}
and  $\Gamma_{i\,\cup\,\mathrm{Pa}(i)}$ is the same function space $\Gamma$  restricted to the subspace $\mathcal{X}_{i \,\cup\, \mathrm{Pa}(i)}$. For example, if $\Gamma = \mathrm{Lip}_b^1(\mathcal{X})$, then  $\Gamma_{i \,\cup\, \mathrm{Pa}(i)} = \mathrm{Lip}_b^1(\mathcal{X}_{i \cup \mathrm{Pa}(i)})$. 

The operators $I_{i,\mathrm{Pa}(i)}$ and the associated projected classes $I_{i,\mathrm{Pa}(i)}[\Gamma]$ provide the mechanism by which a global discriminator objective on $\mathcal{X}$ can be averaged over the conditional structure of $Q$ and restricted to the local family $i\cup \mathrm{Pa}(i)$. They therefore form the bridge between the general variational discrepancies introduced in Section~\ref{sec:variational-representations} and the graph-structured objectives studied next. We now use this notation to show how global discrepancy minimization over $\mathcal{P}^G$ can, under suitable conditions, be related to collections of local family-level problems.

\section{Graph-informed GANs (GiGANs)}
\label{sec:main-results}

In this section, we combine the variational discrepancy framework of Section~\ref{sec:var-gans} with the Bayesian-network operators introduced in Section~\ref{sec:bn-operators}. Our goal is to show that, when the target distribution factorizes according to a known DAG, global discrepancy minimization can be replaced by a principled surrogate involving family-level objectives on the local neighborhoods $i\cup \mathrm{Pa}(i)$. We begin by recalling the generalized subadditivity framework of \cite{ding2021gans}, then explain why this classical notion is not well suited to interpolative divergences, and finally introduce the weaker infimal formulation that underpins our main results for $(f,\Gamma)$-divergences, IPMs, and proximal optimal transport divergences.

In \cite{ding2021gans}, the authors establish generalized subadditivity bounds for a broad class of probability distances and divergences, including the Jensen--Shannon divergence, Wasserstein distance, and nearly all $f$-divergences. For convenience, we recall their definition below. Throughout this section, we fix a graph structure $G$ write $\mathcal{P}^G$ for the class of probability distributions that factorize according to $G$.

\begin{definition}[Generalized subadditivity %
\citep{ding2021gans}]\label{def:inequality_sub}
  Let $P,Q \in \mathcal{P}^G$ be a Bayesian network with underlying graph $G$.
  Given measures of discrepancy $\delta$ and $\delta'$, we say that $\delta$ satisfies $\alpha$-linear \emph{generalized
  subadditivity with error} $\epsilon$ with respect to $\delta'$ on Bayesian
  networks if, for all such $P,Q$, the following inequality holds:
  \begin{equation}\label{inequality_sub}
  \delta'(Q,P)-\epsilon
  \leq \alpha\sum_{i=1}^n
  \delta\left(Q_{i\,\cup\,\mathrm{Pa}(i)},\, P_{i\,\cup\,\mathrm{Pa}(i)}\right).
  \end{equation}
  For the case $\delta=\delta'$, $\epsilon=0$ and $\alpha=1$ we say that $\delta$ satisfies subadditivity on Bayesian networks. 
\end{definition}

By contrast, interpolative divergences presented in Section~\ref{sec:f-Gamma-div}---which have demonstrated strong empirical performance in generative modeling---have not been systematically studied from the viewpoint of subadditivity. A key difficulty is that these discrepancies are defined through primal--dual variational formulations, such as \eqref{def:inf con} and \eqref{eq:f-gamma-variational-rep}, rather than through the classical density-based representation \eqref{def:f_divergence}. In particular, the standard arguments used to prove subadditivity for classical divergences do not transfer directly. The lack of a triangle inequality for divergences is an example; see Theorem 6 and its proof in \cite{ding2021gans}. 

To motivate the need for a different notion, we present two simple calculations. The first recalls the mechanism behind subadditivity for the KL divergences, as proven in \cite[Supplement §8.2]{ding2021gans}. The second shows why the same argument breaks down for interpolative divergences.

To illustrate the first point, consider the simple chain,
\begin{equation}\label{Structure1}
X_1 \longrightarrow X_2 \longrightarrow X_3\,.
\end{equation}
For any $P,Q\in\mathcal P^G$, the key observation is that 
\begin{align}
\mathfrak{D}_{f_{\mathrm{KL}}}(Q\|P)&=E_{Q}\left[\log\left(\frac{dQ}{dP}\right)\right]\label{eq:calc-fKL}\\
&=E_{Q_{12}}\left[\log\left(\frac{dQ_{12}}{dP_{12}}\right) \right]+E_{Q_{23}}\left[\log\left(\frac{dQ_{23}}{dP_{23}}\right)\right]-E_{Q_{2}}\left[\log\left(\frac{dQ_{2}}{dP_{2}}\right)\right]\nonumber\\
&\leq \mathfrak{D}_{f_{\mathrm{KL}}}(Q_{12}\|P_{12})+\mathfrak{D}_{f_{\mathrm{KL}}}(Q_{23}\|P_{23})\nonumber\\
&\leq \mathfrak{D}_{f_{\mathrm{KL}}}(Q_{1}\|P_{1})+\mathfrak{D}_{f_{\mathrm{KL}}}(Q_{12}\|P_{12})+\mathfrak{D}_{f_{\mathrm{KL}}}(Q_{23}\|P_{23}) \nonumber
\end{align}
because $P=\frac{P_{12}P_{23}}{P_2}$ and $Q=\frac{Q_{12}Q_{23}}{Q_2}$. Thus, the logarithm of the likelihood ratio decomposes into local terms, with the overlap corrected by the marginal on $X_2$. As shown in \cite{ding2021gans}, this argument extends to arbitrary DAGs by induction.

The situation is different for interpolative divergences. Consider an $(f,\Gamma)$-divergence with $f=f_{\mathrm{KL}}$ and $\Gamma=\mathrm{Lip}^1_b(\mathcal{X})$, and let $R^{*}$ denote the optimal intermediate measure that solves the infimal convolution \eqref{def:inf con}. By Remark 4.12 of \cite{dupuis2022formulation},
\begin{equation}
    \label{eq:calc-fKL-gamma}
    \mathfrak{D}_{f_{\mathrm{KL}}}^\Gamma(Q\|P)=E_{Q}\left[\log\left(\frac{dR^*}{dP}\right)\right]\,.
\end{equation}
At first sight, \eqref{eq:calc-fKL-gamma} resembles the KL identity \eqref{eq:calc-fKL}. The crucial difference, however, is that the interpolative divergence involves the optimal intermediate measure $R^*$ in place of $Q$. Moreover, \eqref{eq:opt f, gamma} gives,
\begin{equation*}
    dR^{*}\propto (f_{\rm{KL}}^*)'(\gamma^{*})dP= e^{\gamma^{*}-1}dP \,.
\end{equation*}
Using the factorization of $P$, we obtain
\begin{equation*}
    R^{*}(dx_1,dx_2,dx_3)=e^{\gamma^{*}(x_1,x_2,x_3)-1}P_{3 \mid 2}(dx_3)P_{2 \mid 1}(dx_2)P_{1}(dx_1) \,.
\end{equation*}
Now consider the conditional distribution of $X_3$ under $R^{*}$. A direct calculation gives
\[
R^{*}_{3  \mid  \mathrm{Pa}^{*}(3)}(dx_3)
=
\frac{
e^{\gamma^*(x_1,x_2,x_3)-1}\,P_{3\mid 2}(dx_3 \mid x_2)
}{
\int e^{\gamma^*(x_1,x_2,z)-1}\,P_{3\mid 2}(dz \mid x_2)
}.
\]
This expression reveals that, in general, the conditional distribution of $X_3$ under $R^{*}$ depends on both $x_1$ and $x_2$. Hence, 
$$\mathrm{Pa}^{*}(3) \supseteq\{2\}
= \mathrm{Pa}(3)\,.$$ 
In other words, the optimal intermediate measure $R^*$ need not preserve
the conditional independence structure of the original graph.
This prevents the local likelihood terms from being paired according to the families of $G$ as in the first calculation for the classical KL divergence.

These calculations suggest that, for interpolative divergences, classical subadditivity is generally too strong a requirement. In the next subsection, we introduce a weaker notion of subadditivity, formulated at the level of the optimization problem itself, that will allow us to incorporate this structure.

\subsection{Infimal subadditivity for Bayesian networks over a DAG}
\label{sec:infimal-subadditivity-general-bn}

The calculations in the previous section indicate that classical subadditivity
fails for $(f_{\mathrm{KL}},\Gamma)$-divergences. More generally, for any
$f\neq f_{\mathrm{KL}}$, subadditivity continues to fail even under the
closeness assumptions considered in \cite{ding2021gans}. The difficulty stems
from the fact that interpolative divergences are defined through infimal
convolution and variational representations, rather than through the classical
density-based formulation of $f$-divergences. As a result, the arguments used to
establish subadditivity in the classical setting no longer apply. This suggests
that classical subadditivity must be replaced by a weaker notion that is
compatible with optimization over structured model classes. Motivated by this
gap, we introduce the notion of $\alpha$-linear generalized infimal
subadditivity (and additivity) on Bayesian networks, formalized in the next
definition.

\begin{definition}
  \label{def:inequality_inf_sub}
  Let $Q \in \mathcal{P}^G$ be a Bayesian network with underlying graph $G$.
  Given measures of discrepancy $\delta, \delta_1, \dots, \delta_n$, we say that $\delta$ satisfies \emph{$\alpha$-linear generalized infimal
  subadditivity} with respect to $\delta_1, \dots, \delta_n$ on Bayesian
  networks if, for all such $Q$, the following inequality holds:
  \begin{equation}\label{infbounds_2}
  \inf_{P \in \mathcal{M}} \delta(Q,P)
  \leq
  \inf_{P \in \mathcal{M}}\alpha
  \sum_{i=1}^n
\delta_i\left(Q_{i\,\cup\,\mathrm{Pa}(i)}\,, P_{i\,\cup\,\mathrm{Pa}(i)}\right)\,,
  \end{equation}
  with $\mathcal{M} \subseteq \mathcal{P}(\mathcal{X})$.
  For the case $\delta=\delta_1=\cdots=\delta_n$, we say that $\delta$ satisfies infimal subadditivity on Bayesian networks. Furthermore, we say that $\delta$ satisfies \emph{$\alpha$-linear generalized infimal additivity} with respect to $\delta_1, \dots, \delta_n$ on Bayesian
  networks if \eqref{infbounds_2} holds with equality.
\end{definition}

The main classes of probabilities that we consider for $\mathcal{M}$ are $\mathcal{P}(\mathcal{X})$, parametric families, and $\mathcal{P}^G$.
From a purely mathematical perspective, the infimal formulation when $\mathcal{M} = \mathcal{P}(\mathcal{X})$
is trivial, as it is minimized by $P=Q$. However, this viewpoint is not relevant in learning settings, where the target distribution $Q$ is only accessible through samples and the optimization over all probability measures is intractable. Therefore, this is important for applications as even when the target $Q$ factorizes according to $G$, it is generally difficult to enforce the same structural constraint on the optimizer $P^\ast$ during training. In practice, the usefulness of infimal subadditivity is therefore that it replaces a difficult global optimization problem by a surrogate built from family-level discrepancies, without requiring the optimizer itself to be explicitly parameterized as an element of $\mathcal{P}^G$.

We now show that this principle applies to the main discrepancy classes considered in this paper. Under suitable assumptions, $(f,\Gamma)$-divergences, $\Gamma$-IPMs, and proximal optimal transport divergences all admit graph-informed upper bounds in terms of local family objectives, and in the additive regime, these bounds become exact. The next theorem provides a theoretical justification for replacing an intractable global optimization problem with a set of tractable local learning problems. We begin with the interpolative case $(f,\Gamma)$, which provides the basic template for the later results.

\begin{theorem}\label{thm:main2general}
  Let $G$ be an arbitrary DAG and $L,L_1,\dots,L_n>0$. Suppose that $I_{i,\mathrm{Pa}(i)}[\Gamma^L]\subset \Gamma^{L_i}_{i\cup \mathrm{Pa}(i)}$
  for all $i\in V$. Then, $(f,\Gamma)$-divergences satisfy $\frac{1}{n}$-linear generalized infimal subadditivity given in \eqref{infbounds_2} with $\mathcal{M} = \mathcal{P}(\mathcal{X})$ or $\mathcal{P}^G$ and
    $$\delta=\mathfrak{D}_{f}^{\Gamma^L}\;\;\textrm{and }\;\;\delta_i=\mathfrak{D}_{f}^{\Gamma^{L_i}_{i\,\cup\,\mathrm{Pa}(i)}} \quad \text{for all} \; i\in V\,.$$ Moreover, if  $L_i\leq L$ for all $i\in V$, then $(f,\Gamma)$-divergences satisfy $\frac{1}{n}$-linear generalized infimal additivity with $\delta=\mathfrak{D}_{f}^{\Gamma^L}$ and $\delta_i=\mathfrak{D}_{f}^{\Gamma^{L_i}_{i\,\cup\,\mathrm{Pa}(i)}}$ for all  $i\in V$.
\end{theorem}
\begin{proof}
  The proof of the theorem relies on a sequence of data-processing–type inequalities involving the operators $I_{i,\mathrm{Pa}(i)}$, together with the corresponding assumptions $I_{i,\mathrm{Pa}(i)}[\Gamma^L]\subset \Gamma^{L_i}_{i\,\cup \,\mathrm{Pa}(i)}$
  for all $i\in V$. To clarify the main ideas and highlight the role played by these operators, we first present the proof for the simple DAG \eqref{Structure1}, while the proof for a Bayesian network defined on an arbitrary DAG comes next.   Let  $\gamma\in \Gamma^L$, we begin with computing the expectation of $I_{1,2}[\gamma]$ with respect to $Q_{\{1,2\}}$. A direct calculation gives us 
  \begin{eqnarray}
  \label{expectation_IXY}
    E_{Q_{12}}\left[ I_{1,2}[\gamma] \right]
    = \int_{\mathcal{X}_1}\int_{\mathcal{X}_2}\int_{\mathcal{X}_3}\gamma(x_1,x_2,x_3)Q_{12}(dx_1,dx_2)Q_{3 \mid 2}(dx_3)
    = E_{Q}[\gamma].
  \end{eqnarray}
For $P_{12}\in\mathcal{P}(\mathcal{X}_{\{1,2\}})$ (we remind that $\mathcal{X}_{\{1,2\}}=\mathcal{X}_{1}\times \mathcal{X}_{2}$), for any function $f$ as defined in Section~\ref{sec:variational-representations} and for $\nu \in \mathbb{R}$, the convexity $f^*$, Jensen's inequality, and \eqref{eq:general_I_operator} give:\begin{eqnarray}\label{expectation_IXY_Jensen}
    E_{P_{12}}\left[ f^{*}(I_{1,2}[\gamma] - \nu) \right]
    \leq E_{P_{12}}\left[ I_{1,2}[f^*(\gamma - \nu)] \right]
    = E_{P_{12}\otimes Q_{3 \mid 2} }[ f^*(\gamma - \nu) ] \,.
  \end{eqnarray}
where $P_{12}\otimes Q_{3 \mid 2} \in \mathcal{P}_2^G$ is defined in \eqref{def:general_PrMeasQ}. 
  
  Subtracting \eqref{expectation_IXY_Jensen} from \eqref{expectation_IXY} and taking the infimum over $\nu \in \mathbb{R}$ on both sides yields, for all $\gamma \in \Gamma^L$,
  \begin{eqnarray*}
  E_{Q_{12}}\left[I_{1,2}[\gamma]\right] 
    - \inf_{\nu \in \mathbb{R}} 
      \left\{\nu +  E_{P_{12}}\left[f^{*}(I_{1,2}[\gamma] - \nu)\right]\right\} \\
  \geq E_{Q}[\gamma] 
    - \inf_{\nu \in \mathbb{R}} \left\{\nu + E_{P_{12}\otimes Q_{3 \mid 2}}[f^*(\gamma -\nu )] \right\}\,.
  \end{eqnarray*}
  Then taking the supremum over $\gamma \in \Gamma^L$ in both sides of the inequality above, we obtain
  \begin{equation}
  \label{eq:subadditivity-step-1}
  \mathfrak{D}_{f}^{I_{1,2}[\Gamma^L]
  }(Q_{12}\|P_{12})\geq \mathfrak{D}_{f}^{\Gamma^L}(Q\|P_{12}\otimes Q_{3 \mid 2}) \,,
  \end{equation}
  by definition in \eqref{eq:f-gamma-variational-rep}. We similarly obtain 
  \begin{equation}
  \label{eq:subadditivity-step-1_1}
  \mathfrak{D}_{f}^{I_{1}[\Gamma^L]
  }(Q_{1}\|P_{1})\geq \mathfrak{D}_{f}^{\Gamma^L}(Q\|P_{1} \otimes Q_{3 \mid2}\otimes Q_{2 \mid1}) \,,
  \end{equation}
  and 
  \begin{equation}
  \label{eq:subadditivity-step-1_2}
  \mathfrak{D}_{f}^{I_{2,3}[\Gamma^L]
  }(Q_{23}\|P_{23})\geq \mathfrak{D}_{f}^{\Gamma^L}(Q\|P_{23}\otimes Q_{1 \mid 2}) \,.
  \end{equation}
  We define 
  $$c = \inf_{P \in \mathcal M} \mathfrak{D}_f^{\Gamma^L} (Q \| P)\,;$$
  since $\mathcal{P}^G_{i} \subset \mathcal M=\mathcal{P}^G$ or $\mathcal{P}(\mathcal X)$ (and $\mathcal{P}^G\subset \mathcal{P}(\mathcal X)$), in particular it holds that
  \begin{equation}
  \label{eq:graph-constrained-inequality}
  \mathfrak{D}_f^{\Gamma^{L}} (Q \| P_{1} \otimes Q_{3 \mid2}\otimes Q_{2 \mid1} )  + \mathfrak{D}_f^{\Gamma^{L}} (Q \| P_{12} \otimes Q_{3 \mid 2}) +\mathfrak{D}_f^{\Gamma^{L}} (Q\| P_{23} \otimes Q_{1 \mid 2} )\geq 3 c \,.
  \end{equation}
  By using the assumption that $I_{i,\mathrm{Pa}(i)}[\Gamma^L]\subset \Gamma^{L_i}$ for all $i\in V$, we obtain
  \begin{equation}
  \label{eq:subadditivity-step-2}
  \mathfrak{D}_{f}^{\Gamma^{L_1}}(Q_{1}\|P_{1}) + \mathfrak{D}_{f}^{\Gamma^{L_2}}(Q_{12}\|P_{12}) + \mathfrak{D}_{f}^{\Gamma^{L_3}}(Q_{23}\|P_{23}) \geq 3c\,.    
  \end{equation}
Then we consider infimum over $P \in \mathcal{P}(\mathcal{X})$ on both sides yields \eqref{infbounds_2}. 

  For the other direction, it is enough to observe that when $I_{i,\mathrm{Pa}(i)}[\Gamma]\subset \Gamma^{L_i}$ for all $i\in V$
  \begin{equation}\label{ineq:other_direction}
  \mathfrak{D}_{f}^{I_{i,\mathrm{Pa}(i)}[\Gamma^L]}(Q_{i\,\cup\,\mathrm{Pa}(i)}\|P_{i\,\cup\,\mathrm{Pa}(i)}) \leq \mathfrak{D}_{f}^{\Gamma^{L_i}}(Q_{i\,\cup\,\mathrm{Pa}(i)}\|P_{i\,\cup\,\mathrm{Pa}(i)})\leq \mathfrak{D}_{f}^{\Gamma^L}(Q\|P)\,,
  \end{equation}
  where the last inequality comes from the fact that $L_i\leq L$ for all $i\in V$. Then the strict additivity follows. 

  For a general DAG $G$, it suffices to observe that for any $i\in V$  \begin{equation*}E_{Q_{i\,\cup\,\mathrm{Pa}(i)}}\left[ I_{i,\mathrm{Pa}(i)}[\gamma] \right]
    =\int_{\mathcal{X}}\gamma(\bm x)Q(d\bm x) = E_{Q}[\gamma]
    \end{equation*} 
   and $$E_{P_{i\,\cup\,\mathrm{Pa}(i)}}[f^*(I_{i,\mathrm{Pa}(i)}[\gamma]-\nu)]\leq E_{\tilde P}[f^*(\gamma-\nu)]$$
   where $\tilde{P}= Q_{V^{(i)} \mid i\,\cup\,\mathrm{Pa}(i)}\otimes\, P_{i\,\cup\,\mathrm{Pa}(i)}\otimes \prod_{j=i+1}^n Q_{j \mid\mathrm{Pa}(j)} \in\mathcal{P}_i^G$.   The other direction given in \eqref{ineq:other_direction} remains unchanged, and combining both inequalities completes the proof. When the infimum is be taken over $P\in\mathcal{P}^G$, the proof  follows the same steps as above.
\end{proof}

\begin{remark}[Computational tractability]
 In practice enforcing the graph structure $\mathcal{P}^G$ directly in the generator class is challenging. Therefore, the theorem allows to construct tractable surrogate objectives that approximate or control the global divergence without requiring a graphical factorization on the model $P$. For this reason, we retain the larger space $\mathcal{P}(\mathcal{X})$. The graph structure is used to define a collection of local objectives that collectively control the global divergence.
\end{remark}
 
Theorem~\ref{thm:main2general} gives a variational justification for replacing a monolithic discriminator acting on the full joint space with a graph-informed collection of family-level discriminators aligned with $G$. In practice, if the generator is restricted to $P_\theta \in \mathcal{P}(\mathcal{X})$, one may train by minimizing the graph-informed surrogate
\begin{equation}
  \label{eq:local-surrogate}
  \mathcal{L}_{\mathrm{loc}}(\theta)
  = \frac1n\sum_{i=1}^n \sup_{\gamma_i\in\Gamma^{L_i}_{i\cup \mathrm{Pa}(i)}}
   \!\!\!\!\! \big\{E_{Q}[\gamma_i(\boldsymbol{x}_{i\cup \mathrm{Pa}(i)})]
  -\inf_{\nu\in\mathbb R}\big(\nu+ E_{P_\theta}[f^*(\gamma_i(\boldsymbol{x}_{i\cup \mathrm{Pa}(i)})-\nu)]\big)\big\}.
\end{equation}
Thus the attained value of $\mathcal{L}_{\mathrm{loc}}(\theta)$ provides an explicit upper bound (certificate) on the best achievable global interpolative divergence within $\mathcal{P}(\mathcal{X})$. In the additive regime, minimizing $\mathcal{L}_{\mathrm{loc}}$ is equivalent to minimizing the corresponding global objective.

To illustrate the meaning of Theorem~\ref{thm:main2general} before turning to the corresponding IPM and proximal optimal transport results, we consider an application to a synthetic data set in which the graph structure is known. The purpose of this example is twofold: first, to assess whether the graph-informed surrogate tracks the best achievable global objective, and second, to test whether graph alignment improves recovery of the underlying conditional structure.

\begin{figure}[htbp]
\centering
  \includegraphics[width=0.3\textwidth]{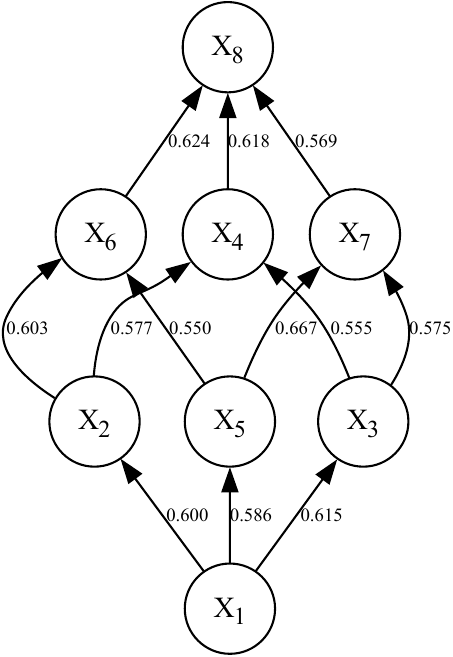}
\caption{The linear-Gaussian Bayesian network given by the Hasse diagram of the subset lattice for the $8$ subsets of a three-element set. The edge weights from node $j$ to $i$ are the parent coefficients $\phi_{ji} \sim N(\mu_\phi = 0.6, \sigma^2_\phi = 0.05^2)$ in \eqref{eq:linear-structured-Gaussian}, which are fixed once across all experiments.}
\label{fig:exm-structured-gaussian}
\end{figure}

We consider a linear-Gaussian Bayesian network on eight variables, shown in Figure~\ref{fig:exm-structured-gaussian}, defined by
\begin{equation}
  \label{eq:linear-structured-Gaussian}
  X_i = \sum_{j \in \mathrm{Pa}(i)} \phi_{ji} X_{j} + \epsilon_{i} \,, \qquad i \in \{1, \dots, 8\}\,,
\end{equation} 
where $\epsilon_i \sim N(0,\sigma^2)$ with $\sigma=0.5$ and the root node satisfies $X_1 \sim N(0,1)$. This model provides a direct testbed for Theorem~\ref{thm:main2general}, since the correct family structure is known through the parent coefficients $\{\phi_{ji}\}$ (full details of the experiment can be found in Section~\ref{sec:experiments}).

\begin{figure}[htb]
\centering
\begin{subfigure}[t]{0.49\textwidth}
  \centering
  \includegraphics[width=\textwidth]{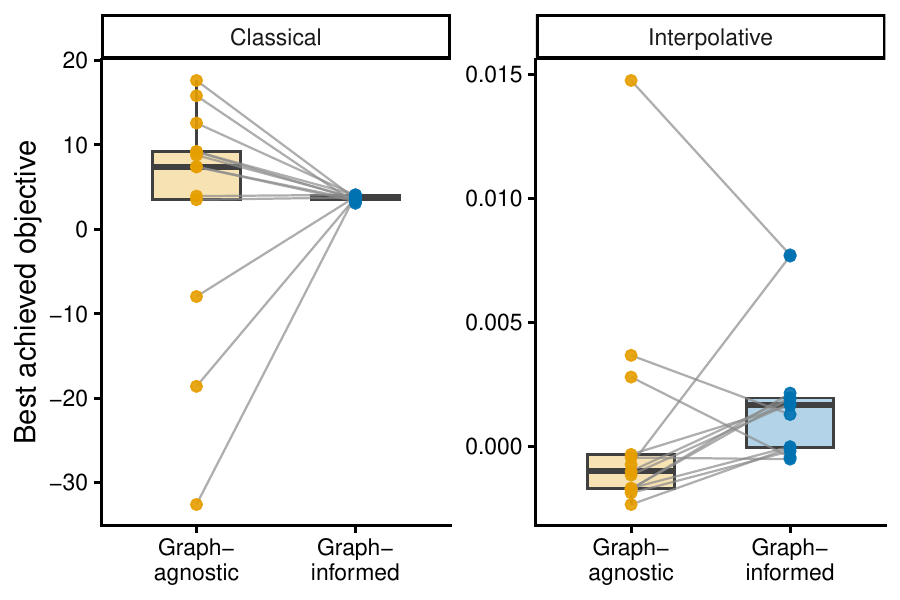}
  \caption{Best achieved variational objective}
  \label{fig:exm-structured-gauss:upper-bound}
\end{subfigure}\hfill
\begin{subfigure}[t]{0.49\textwidth}
  \centering
  \includegraphics[width=\textwidth]{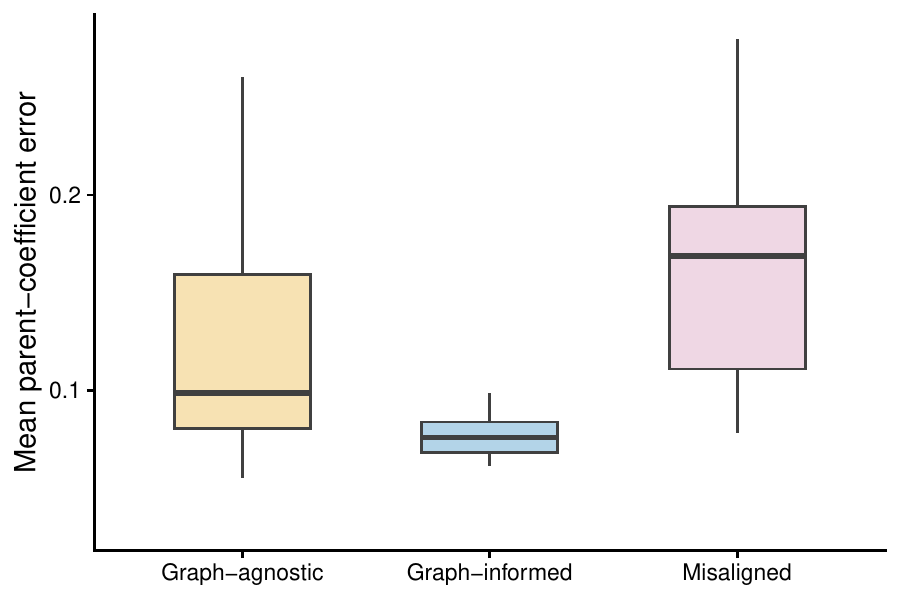}
  \caption{Parent-coefficient recovery} 
  \label{fig:exm-structured-gauss:local-recovery}
\end{subfigure}
\caption{Objective-level and structural diagnostics for the structured linear-Gaussian model~\eqref{eq:linear-structured-Gaussian} under graph-agnostic and graph-informed adversarial training; see Figure~\ref{fig:exm-structured-gaussian} for the underlying Hasse network. (a) Best achieved validation variational objective, i.e., the minimum variational loss attained over training epochs for the validation set. Values closer to zero indicate a closer fit, and gray line segments pair runs with the same random seed. In the interpolative setting, the graph-informed multi-discriminator surrogate attains values close to those of the graph-agnostic monolithic objective, consistent with the interpretation of Theorem~\ref{thm:main2general}. (b) Mean $L_2$ error in the recovered parent coefficients, averaged over non-root nodes, with lower values indicating better recovery of the underlying conditional structure. Here, the graph-informed construction aligned with the true families outperforms both the graph-agnostic baseline and the misaligned multi-discriminator construction. Boxplots summarize $13$ paired runs per model.}
\label{fig:exm-structured-gauss}
\end{figure}

Figure~\ref{fig:exm-structured-gauss:upper-bound} compares a graph-agnostic GAN, whose monolithic discriminator acts on the full joint distribution, with a GiGAN built from family-level discriminator classes on $\{i\cup \mathrm{Pa}(i)\}_{i=1}^8$. We report both the classical setting, based on unrestricted $f$-divergence objectives, and the interpolative setting, based on constrained $(f,\Gamma)$-divergence objectives when $f=f_{\mathrm{KL}}$ and $\Gamma= \mathrm{Lip}_b$. In the interpolative setting, the graph-informed surrogate closely tracks the best achieved global objective and exhibits substantially lower variability across seeds. By contrast, in the classical setting the graph-agnostic objective is markedly more variable, reflecting the greater difficulty of the full-joint discrimination problem.

To connect objective values with the structural fidelity of model predictions, Figure~\ref{fig:exm-structured-gauss:local-recovery} reports recovery of the parent coefficients $\{\phi_{ji}\}$ from generated samples for the experiments based on the interpolative divergences. We also include a misaligned multi-discriminator baseline with matched family sizes but randomly assigned parent sets. The graph-informed approach achieves lower estimation error and smaller variability than both the graph-agnostic and misaligned multi-discriminator baselines. This is consistent with the interpretation of Theorem~\ref{thm:main2general}, that when the local discriminator classes are matched to the true graph structure, they control precisely the family-level distributions needed to recover the conditional relationships.

Taken together, these diagnostics show that the graph-informed surrogate is not only theoretically justified but also practically meaningful: in the interpolative regime it provides a stable proxy for the global objective, while preserving the structural information encoded by the Bayesian network. We now return to the theory and state the analogous result for $\Gamma$-IPMs.

\medskip
\begin{theorem}\label{thm:main_IPM}
  Let $G$ be an arbitrary DAG and $L,L_1,\dots,L_n>0$. Suppose that $I_{i,\mathrm{Pa}(i)}[\Gamma^L]\subset \Gamma^{L_i}_{i\cup \mathrm{Pa}(i)}$
  for all $i\in V$. Then, IPMs satisfy $\frac{1}{n}$-linear generalized infimal subadditivity given in \eqref{infbounds_2} with $\mathcal{M} = \mathcal{P}(\mathcal{X})$ or $\mathcal{P}^G$ and
  $$\delta=W^{\Gamma^L} \;\;\textrm{and} \;\;\delta_i=W^{\Gamma^{L_i}_{i\,\cup\,\mathrm{Pa}(i)}} \quad \text{for all} \; i\in V\,.$$
  
  Moreover, if  $L_i\leq L$ for all $i\in V$, then IPMs satisfy $\frac{1}{n}$-linear generalized infimal additivity with the same $\delta$ and $\delta_i$ as above.
\end{theorem}

\begin{proof}
  The proof follows the same arguments as that of Theorem~\ref{thm:main2general} and is therefore omitted.
\end{proof} 

The same principle extends beyond variational divergences and $\Gamma$-IPMs to transport-based discrepancies. We next show that proximal optimal transport divergences also admit an infimal subadditivity structure over Bayesian-network families.

\begin{theorem}\label{thm:pOTd} 
Let $c$ be any cost function on $\mathcal{X}$. Then, proximal optimal transport divergences satisfy $\frac{1}{n}$-linear generalized infimal subadditivity given in \eqref{infbounds_2} with $\mathcal{M} = \mathcal{P}(\mathcal{X})$ or $\mathcal{P}^G$ and
\begin{equation}
\delta=\mathfrak{D}_{\rm{KL},\varepsilon}^{c} \;\;\textrm{and }\;\;\delta_i=\mathfrak{D}_{\rm{KL},\varepsilon}^{c_i}, \quad \text{for all} \; i\in V\,,
\end{equation}
where  
\begin{eqnarray*}c_i(\bm x_{i\,\cup\,\mathrm{Pa}(i)},\bm x'_{i\,\cup\,\mathrm{Pa}(i)})&=&\int_{\mathcal X_{A}\times \mathcal X_A} c(\bm x,\bm x')\prod_{\bm y=\bm x,\bm x'}\Big(Q(d\bm y_{V^{(i-1)}}|\bm y_{\mathrm{Pa}(i)})\prod_{j=i+1}^n Q(dy_j|\bm y_{\mathrm{Pa}(j)})\Big)\,,\end{eqnarray*} 
where $A=[i\,\cup\,\mathrm{Pa}(i)]^c$. 
\end{theorem}
\begin{proof}
Following the same strategy as in Theorem \ref{thm:main2general}, for the simple DAG \eqref{Structure1} we obtain
\begin{equation}
\begin{split}
&\sup_{\phi\oplus\psi\leq c}
\left\{ E_{Q}[\phi]  -\varepsilon \log E_{P_{\{1,2\}}\otimes Q_{3 \mid 2}}[e^{-\psi/\varepsilon}]\right\}\nonumber\\
&\qquad \leq \sup_{\phi\oplus\psi\leq c}
\left\{ E_{Q_{12}}[I_{1,2}[\phi]]  -\varepsilon \log E_{P_{12}}[e^{-I_{1,2}[\psi]/\varepsilon}]\right\}\nonumber\\
&\qquad \leq \sup_{\tilde\phi\oplus\tilde\psi\leq c_2}
\left\{ E_{Q_{12}}[\tilde\phi]  -\varepsilon \log E_{P_{12}}[e^{-\tilde\psi/\varepsilon}]\right\}\nonumber\\
\end{split}
\end{equation}
where the second inequality comes from the fact that \begin{equation*}
    \begin{split}
 I_{1,2}[\phi](x_1,x_2)&+I_{1,2}[\psi] (x_1',x_2')=\int_{\mathcal{X}_3 } \phi(x_1,x_2,x_3) Q_{3 \mid 2}(dx_3)+\int_{\mathcal{X}_3 } \psi(x_1',x_2',x_3') Q_{3 \mid 2}(dx_3')\nonumber\\
 &=\int_{\mathcal{X}_3 }\int_{\mathcal{X}_3 } (\phi(x_1,x_2,x_3)+ \psi(x_1',x_2',x_3') )Q(dx_3|x_2)Q(dx'_3|x'_2)\nonumber\\
 &\leq\int_{\mathcal{X}_3 }\int_{\mathcal{X}_3 } c(\bm x,\bm x')Q(dx_3|x_2)Q(dx'_3|x'_2)
    \end{split}
\end{equation*} 
Then, the result follows as in proof of Theorem \ref{thm:main2general}. 
\end{proof}

From the perspective of adversarial learning, the preceding theorems justify replacing a graph-agnostic GAN with a monolithic discriminator by a GiGAN with family-level discriminators. This is not only a computational simplification; since GAN training amounts to a learned two-sample testing problem \citep{goodfellow2014generative,nowozin2016f}, localization decomposes a high-dimensional discrimination task into a collection of lower-dimensional subproblems that are easier to optimize and diagnose. Theorems~\ref{thm:main2general}, \ref{thm:main_IPM}, and \ref{thm:pOTd} show that these local games define principled surrogates for the corresponding graph-constrained global objectives.

In Section \ref{subsec_examples_adm_sets}, we provide examples 
and discuss conditions under which the condition in Theorem~\ref{thm:main2general} which we henceforth denote by
\[
\mathbf{(C)}\qquad I_{i,\mathrm{Pa}(i)}[\Gamma^L]\subset \Gamma^{L_i}_{i\cup \mathrm{Pa}(i)}
\quad \text{for all } i\in V,
\]
holds, as well as its 
interpretation from the perspective of generative modeling.
When the condition $\mathbf{(C)}$
is violated, we can still obtain meaningful lower and upper bounds for
$(f,\Gamma)$-divergences and IPMs.
Specifically, for any $P,Q\in\mathcal{P}^G$, we have the
following lower bound:
\begin{equation}\label{infbounds_1}
\frac{1}{n}\sum_{i=1}^n
\mathfrak{D}_{f}^{\Gamma^L_{i\cup \mathrm{Pa}(i)}}
\!\left(Q_{i\cup \mathrm{Pa}(i)}\middle\|P_{i\cup \mathrm{Pa}(i)}\right)
\;\leq\;
\mathfrak{D}_{f}^{\Gamma^L}(Q\|P)
\end{equation}
and for any $Q\in\mathcal{P}^G$, we obtain
\begin{equation}\label{infbounds_2_upper}
\inf_{P\in\mathcal P(\mathcal{X})} \mathfrak{D}_{f}^{\Gamma^L}(Q\|P)
\;\leq\;
\inf_{P\in\mathcal P(\mathcal{X})}\frac{1}{n}
\sum_{i=1}^n
\mathfrak{D}_{f}^{I_{i,\mathrm{Pa}(i)}[\Gamma^L]}
\!\left(Q_{i\cup \mathrm{Pa}(i)}\|P_{i\cup \mathrm{Pa}(i)}\right).
\end{equation}
The same lower and upper bounds  hold for IPMs. This follows directly from the
monotonicity property
$\mathfrak{D}_{f}^{\Gamma_{i\cup \mathrm{Pa}(i)}}(Q_{i\cup \mathrm{Pa}(i)}\|P_{i\cup \mathrm{Pa}(i)})
\leq \mathfrak{D}_{f}^{\Gamma}(Q\|P)$ for all $i\in V$. The proof of the upper bound is already contained in the proof of
Theorem~\ref{thm:main2general}.  
\medskip

\begin{remark}[Graph-informed $\Gamma$]\label{rem:structured_gamma}A special case where the discriminator class is explicitly aligned with the graph
structure, as each component depends only on a node and its parent set is given by $$\Gamma=\oplus_{i=1}^n\Gamma_{i\,\cup\,\mathrm{Pa}(i)} =\left\{\sum_{i=1}^n\gamma_i(x_i,\bm x_{\mathrm{Pa}(i)}): \gamma_i\in\Gamma_{i\,\cup\,\mathrm{Pa}(i)} \textrm{ for all } i\in V\right\}\,.$$ Each element of the  discriminator space decomposes additively over local neighborhoods,
which makes the subadditivity property for IPMs immediate:
$$W^{\oplus_{i=1}^n\Gamma_{i\,\cup\,\mathrm{Pa}(i)}}
\left(Q_{i\,\cup\, \mathrm{Pa}(i)}\middle\|P_{i\,\cup\, \mathrm{Pa}(i)}\right)\leq \sum_{i=1}^n W^{ \Gamma_{i\,\cup\,\mathrm{Pa}(i)}}
\left(Q_{i\,\cup\, \mathrm{Pa}(i)}\middle\|P_{i\,\cup\, \mathrm{Pa}(i)}\right).$$
\end{remark}

\begin{remark}[Tighter bound]
\label{rem:tighter_bound}

Although we have shown that the subadditivity-based upper bound for a Bayesian network
depends explicitly on the underlying DAG $G$ with $n$ nodes—through the induced
collection of local neighborhoods $\{ i \cup \mathrm{Pa}(i) \}_{i=1}^n$, which in turn
determines the construction of the local discriminators—the choice of these
neighborhoods is not unique. In particular, the local neighborhoods may be enlarged,
leading to an alternative subadditivity-based upper bound that operates at a coarser
level of granularity. Such bounds may typically be tighter, as they incorporate richer local dependencies. This flexibility arises because the proof of Theorem~\ref{thm:main2general} hinges on
establishing a data-processing–type inequality. Once this inequality is available, it
allows the divergence to be further upper bounded by divergences between distributions
defined on the remaining neighborhoods induced by the graph $G$. Consequently, the
selection of local neighborhoods can be interpreted as an artificial hyperparameter
that controls the trade-off between model complexity and tightness of the bound. In
Child Bayesian network in Section~\ref{sec:experiments:real-bn}, we illustrate this idea by replacing certain local neighborhoods
with enlarged families that include additional ancestors. 
\end{remark}

The preceding results are abstract until condition {\bf (C)} can be verified for concrete discriminator classes. We therefore next discuss representative choices of $\Gamma$ for which the projection operators $I_{i,\mathrm{Pa}(i)}$ preserve the relevant function-space structure, thereby making the theory applicable in practice.

\subsection{Examples of admissible sets satisfying condition {\bf (C)}}
\label{subsec_examples_adm_sets}

In this section, we provide examples of admissible sets that satisfy condition {\bf (C)}:  bounded measurable functions and Lipschitz-bounded functions.  These examples respectively recover the unrestricted $f$-divergence setting, and Wasserstein-type and Lipschitz-constrained discriminators. 

\paragraph{The space of bounded measurable functions $\mathcal{M}_{b}(\mathcal{X})$.}
Let $\Gamma=\mathcal{M}_{b}(\mathcal{X})$. Then, for $i\in V$,  $I_{i,\mathrm{Pa}(i)}[\Gamma]$ is subset of the restricted version of $\mathcal{M}_{b}(\mathcal{X})$ to $\mathcal X_{i\,\cup\,\mathrm{Pa}(i)}$, i.e.,  $I_{i,\mathrm{Pa}(i)}[\mathcal{M}_{b}(\mathcal{X})]\subset \mathcal{M}_{b}(\mathcal{X}_{i\,\cup\,\mathrm{Pa}(i)})$. In Theorem 11 of \cite{ding2021gans}, by assuming that $P,Q$ are two-sided $\epsilon$-close (see Definition 2 in \cite{ding2021gans}), they prove that any $f$-divergence whose $f(\cdot)$ is continuous on $(0,\infty)$ and twice differentiable at $1$ with $f''(1) > 0$ satisfies $\alpha$-linear subadditivity (see Definition \ref{def:inequality_sub}).  By applying Theorem \ref{thm:main2general}, we prove that $f$-divergences satisfy $\frac{1}{n}$-linear infimal additivity. 

\begin{theorem}\label{thm:additivity}
  Let $G$ be an arbitrary DAG. Then, $f$-divergences satisfy $\frac{1}{n}$-linear generalized infimal additivity with $\mathcal M=\mathcal P(\mathcal X)$ or $\mathcal P^G$
\begin{equation}\label{infequality_measurable functions}
\inf_{P\in\mathcal{P}(\mathcal{X})} \mathfrak{D}_{f}(Q\|P)
=
\inf_{P\in\mathcal{P}(\mathcal{X})}\frac{1}{n}
\sum_{i=1}^n
\mathfrak{D}_{f}\left(Q_{i\cup \mathrm{Pa}(i)}\middle\|P_{i\cup \mathrm{Pa}(i)}\right)
\end{equation}
\end{theorem}
\begin{proof}
It is enough to observe that $$\mathfrak{D}_{f}^{\mathcal{M}_{b}(\mathcal{X}_{i\,\cup\,\mathrm{Pa}(i)})}
\!\left(Q_{i\cup \mathrm{Pa}(i)}\middle\|P_{i\cup \mathrm{Pa}(i)}\right)=\mathfrak{D}_{f}
\left(Q_{i\cup \mathrm{Pa}(i)}\middle\|P_{i\cup \mathrm{Pa}(i)}\right)\,.$$
\end{proof}

\paragraph{The space of Lipschitz bounded functions ${\rm{Lip}}^L_b(\mathcal{X})$, $L>0$.} Suppose that $\mathcal X$ is a metric space. The standard metric choice is   
$$d(\bm{x},\bm{x}')=\sum_{i=1}^n d_{\mathcal{X}_i}(x_i,x'_i),$$ for $\bm{x}=(x_1,\dots,x_n)$ and $\bm{x}'=(x'_1,\dots,x'_n)$. Let $\Gamma^{L}=\textrm{Lip}^{L}_b(\mathcal{X})$. In this case, condition {\bf (C)} can be satisfied if $Q_{V^{(i-1)} \mid\mathrm{Pa}(i)}\otimes \prod_{j=i+1}^n Q_{j \mid\mathrm{Pa}(j)}$ viewed as a function of $\bm x_{i\,\cup\,\mathrm{Pa}(i)}$, is Lipschitz-continuous with respect to 1-Wasserstein distance  \citep{bartl2025fast, piening2024paired,altekruger2023conditional, benezet2024learning}. We prove this result in Proposition~\ref{prop:Lip_continuity for conditional Q}.  

In generative modeling, the assumption that $Q_{V^{(i-1)} \mid\mathrm{Pa}(i)}\otimes \prod_{j=i+1}^n Q_{j \mid\mathrm{Pa}(j)}$ viewed as a function of $\bm x_{i\,\cup\,\mathrm{Pa}(i)}$, is Lipschitz-continuous with respect to 1-Wasserstein distance (also   given in \eqref{ass:Lip1}), can be interpreted as follows: Let $(\mathcal Z,\mathcal F,\nu)$ be a probability space, $\mathcal X_{\mathrm{Pa}(i)}\subset\mathbb R^m$ and $g:\mathcal X_{\mathrm{Pa}(i)}\times \mathcal Z\to\mathbb R^d$. We assume that 
for every $\bm x_{\mathrm{Pa}(i)}\in\mathcal X_{\mathrm{Pa}(i)}$, the map $z\mapsto g(\bm x_{\mathrm{Pa}(i)},z)$ is measurable 
and  $g(\bm x_{\mathrm{Pa}(i)},\cdot)_{\#}\nu\in\mathcal P_1(\mathbb R^d)$, and for any $\bm x_{\mathrm{Pa}(i)},\bm x'_{\mathrm{Pa}(i)}\in\mathcal X_{\mathrm{Pa}(i)}$ and all $z\in\mathcal Z$, 
$\|g(\bm x_{\mathrm{Pa}(i)},z)-g(\bm x'_{\mathrm{Pa}(i)},z)\|\leq Cd_{\mathcal{X}_{\mathrm{Pa}(i)}}(\bm x_{\mathrm{Pa}(i)},\bm x'_{\mathrm{Pa}(i)})$.
Then, a measurable generator with finite first moment and uniform Lipschitz dependence on the conditioning random vector induces a $W_1$-Lipschitz family of conditional distributions, see \cite{arjovsky2017wasserstein, papamakarios2021normalizing}.
This is proved in Appendix \ref{supplem}.

Therefore, such an assumption (see  \eqref{ass:Lip1})
expresses a quantitative stability property of the model: small perturbations of the conditioning variables induce proportionally small changes in the corresponding conditional distributions in $1$-Wasserstein distance. This stability naturally arises when conditional distributions are represented as the pushforwards of a reference noise distribution through generators that are Lipschitz in their conditioning variables.

\section{Infimal subadditivity in action}
\label{sec:experiments}

Section~\ref{sec:main-results} shows that when the target distribution factorizes according to a known Bayesian network over a graph $G$, the corresponding global objective function can be related to a sum of lower-dimensional objective functionals defined on graph-informed families. In adversarial learning terms, this suggests replacing a graph-agnostic GAN with a monolithic discriminator by a GiGAN with localized discriminators acting on families of variables. The aim of this section is to show how the graph-informed strategy behaves in practice across a range of structured target distributions. In particular, we find that GiGANs often improve optimization stability and the recovery of the underlying conditional structure, while also revealing trade-offs between family-level and global distributional discrepancy, and between predictive accuracy and computational cost. 

We study four case studies. The first two are synthetic continuous examples. The structured linear-Gaussian model, introduced in Section~\ref{sec:main-results}, serves as a controlled testbed in the sense that the graph structure is known and the effect of localization can be examined directly. In what follows, we complement the earlier objective-level results with additional diagnostics for global versus local error and with an experiment on the role of the Lipschitz constant. The second synthetic example is a ball-throwing time-series model \citep[see, e.g., §8.1]{ding2021gans}, which we use to examine training dynamics and convergence in the structured continuous setting. 

The remaining two case studies use real discrete Bayesian networks, namely, \textsc{Child} from \cite{spiegelhalter1992learning} and \textsc{Earthquake} from \cite{pearl1988probabilistic}. We use \textsc{Child} to study structural fidelity and the effect of family size in a medium-sized network, and \textsc{Earthquake} to assess robustness across many runs in a smaller network. 

The remainder of this section is organized as follows. We first describe the common training framework shared across the experiments, including the localized discriminator construction, objective functions, and validation metrics. We then present the synthetic case studies, followed by the real-world Bayesian network experiments. 

\subsection{Common experimental framework}
\label{sec:algorithm}

All four case studies train a generator to match a target distribution $Q$ through a variational min--max objective of the form,
\begin{equation}
    \label{eq:var-learning-f-Gamma}
    \inf_{P_\theta} \mathfrak{D}_{f}^\Gamma (Q \| P_\theta) = \inf_{P_\theta} \sup_{\gamma_\eta \in \Gamma} \left\{ E_{Q} [ \gamma_\eta(x)] - \inf_{\nu \in \mathbb{R}} \left\{\nu + E_{P_\theta} [f^* (\gamma_\eta(x)-\nu)] \right\} \right\}\,, 
\end{equation}
where $\theta$ and $\eta$ denote the generator and discriminator parameters, respectively. In the graph-agnostic setting, a single monolithic discriminator acts on the full sample vector. In the graph-informed setting, the discriminator is replaced by a collection of localized discriminators acting on lower-dimensional families of variables.

\paragraph{Localized discriminators.} 
Recall that $V$ denotes the full set of vertices in the DAG and hence $|V| = n$ denotes the number of variables and that $\boldsymbol{x}$ denotes a full sample from either the target or generated distribution. Given a collection of families $F_1, \dots, F_n$, with $F_i \subset V$ for all $i \in V$, we associate to each $F_i$ a discriminator
\begin{equation*}
    \gamma^i_{\eta_i} = \gamma^i(\cdot; \eta_i): \mathcal{X}_{F_i} \to \mathbb{R}\,, \qquad i = 1, \dots, n\,,
\end{equation*}
which acts only on subvectors $\boldsymbol{x}_{F_i}$. In the graph-informed experiments, these families are chosen to align with the structure of the Bayesian network. In the simplest case, they are the child--parent families $F_i = \{i\} \cup \mathrm{Pa}(i)$. For discrete Bayesian networks, we generate a dummy encoding of each variable; if $e(v)$ is the block of encoded coordinates for variable $v$, then we define encoded families $\bar{F}_i = \cup_{v\in F_i} e(v)$ induced by $F_i$ and the discriminators act on these encoded families $\gamma^i:\mathcal{X}_{\bar{F}_i}\to\mathbb{R}$.

\paragraph{Local objectives.}
For each localized discriminator, we define an objective functional comparing the corresponding marginals of real and generated samples. We consider two forms. For continuous experiments based on the KL-divergence, we use the Donsker--Varadhan representation, 
\begin{equation}
\label{eq:local-obj-dv}
    \mathcal{J}^{i}_{\mathrm{DV}}(\theta, \eta_i) = E_{Q} [\gamma^{i}(\boldsymbol{x}_{F_i};\eta_i)] - \log E_{P_\theta} [\exp(\gamma^{i}(\boldsymbol{x}_{F_i};\eta_i))]\,.
\end{equation}
For the discrete experiments, we use the $f$-GAN form \cite{nowozin2016f},
\begin{equation}
\label{eq:local-obj-f-gan}
    \mathcal{J}_f^{i}(\theta, \eta_i) = E_{Q} [\gamma^{i}(\boldsymbol{x}_{\bar{F}_i};\eta_i)] - E_{P_\theta}[f^* (\gamma^{i}(\boldsymbol{x}_{\bar{F}_i};\eta_i))]\,,
\end{equation}
where $f^*$ is the convex conjugate corresponding to the chosen $f$-divergence (see Section~\ref{sec:f-div}). In the latter case, the output activation of each discriminator is chosen so that its values lie in the effective domain of $f^*$. Note that for the KL-divergence, \eqref{eq:local-obj-dv} is a special case of \eqref{eq:var-learning-f-Gamma} after optimizing over $\nu$, whereas the $f$-GAN formulation in \eqref{eq:local-obj-f-gan} corresponds to fixing $\nu$ and therefore changes the gradients and training dynamics.

\paragraph{Global graph-informed objective.} 
Training is performed by summing the local objectives,
\begin{equation}
    \label{eq:sum-local-obj}
    \mathcal{J}(\theta, \eta_1, \dots, \eta_n) = \sum_{i=1}^{n} w_i\, \mathcal{J}^{i}(\theta, \eta_i) \,.
\end{equation}
When all families correspond to the child--parent families, the choice $w_i = 1/n$ matches the averaged form suggested by the infimal subadditivity results up to a constant normalization. Algorithm~\ref{alg:many-critic-training} summarizes the resulting training scheme.

\begin{algorithm}[htb]
\caption{GiGAN training with localized discriminators}
\label{alg:many-critic-training}
\DontPrintSemicolon
\KwIn{Target $Q$; latent noise $P_Z$; generator $\mathrm{Gen}(\cdot; \theta)$;  batch~size~$B$; families~$F_1, \dots, F_n$; local~discriminators~$\gamma^{1}(\cdot; \eta_1), \dots, \gamma^n(\cdot; \eta_n)$; learning~rates~$\ell_{\eta},\ell_{\theta}$; weights~$w_1,\dots, w_n$.}
\KwChoice{Objective form (A) DV-KL or (B) $f$-GAN.}
\medskip 
\For{training steps $t=0,1,2,\dots$}{
  Sample minibatches $\{\boldsymbol{x}^{(k)}\}_{k=1}^B \sim Q$ and $\{\boldsymbol{z}^{(k)}\}_{k=1}^B \sim P_Z$\;
  Form generated minibatch $\{\tilde{\boldsymbol{x}}^{(k)}\}_{k=1}^B$:\;
  \eIf{categorical data}{
  \tcp{Generator outputs logits for each categorical variable block; relaxed samples are formed via Gumbel-softmax} %
      $\tilde{\boldsymbol{x}}^{(k)} \leftarrow \mathrm{GumbelSoftmax}_\tau (\mathrm{Gen}(\boldsymbol{z}^{(k)}; \theta^t)) \quad \text{for } k=1,\dots,B$\;
    }{
    $\tilde{\boldsymbol{x}}^{(k)} \leftarrow \mathrm{Gen}(\boldsymbol{z}^{(k)}; \theta^t) \quad \text{for } k=1,\dots,B$\;
    }  
    
  Calculate objectives and update parameters:\;
  \For{$i=1,\dots, n$}{
    \uIf{(A) DV--KL}{
      $\widehat{\mathcal{J}}^{i} \leftarrow \frac1B\sum_{k=1}^B \gamma^{i}(\boldsymbol{x}_{F_i}^{(k)}; \eta_i^t)
      \;-\;\log\!\left(\frac1B\sum_{k=1}^B \exp(\gamma^{i}(\tilde{\boldsymbol{x}}_{F_i}^{(k)}; \eta_i^t))\right)$ \tcp*[r]{see \eqref{eq:local-obj-dv}}
    }\ElseIf{(B) $f$-GAN}{
      $\widehat{\mathcal{J}}^{i} \leftarrow \frac1B\sum_{k=1}^B \gamma^{i}(\boldsymbol{x}_{\bar{F}_i}^{(k)}; \eta_i^t)
      \;-\;\frac1B\sum_{k=1}^B f^*\!\left(\gamma^{i}(\tilde{\boldsymbol{x}}_{\bar{F}_i}^{(k)}; \eta_i^t)\right)$ \tcp*[r]{see \eqref{eq:local-obj-f-gan}}
      \tcp{activation is chosen so that discriminator outputs lie in the effective domain of $f^*$}
    }
  }
  $\widehat{\mathcal{J}} \leftarrow \sum_{i=1}^n w_i\,\widehat{\mathcal{J}}^{i}$ \tcp*[r]{see \eqref{eq:sum-local-obj}}

  \tcp{One discriminator step per minibatch}
  $\eta_i^{t+1} \leftarrow \eta_i^t + \ell_{\eta} \nabla_{\!\eta_i} \widehat{\mathcal{J}} (\theta^t,\eta_1^t,\dots,\eta_n^t)\quad \text{for } i = 1, \dots, n$\;
  \tcp{One generator step per minibatch}
  $\theta^{t+1} \leftarrow \theta^t - \ell_{\theta} \nabla_{\!\theta}\widehat{\mathcal{J}} (\theta^t,\eta_1^t,\dots,\eta_n^t)$\;
  Record validation and runtime diagnostics\;
}
\end{algorithm}

\paragraph{Continuous and discrete generators.}
For continuous data, the generator outputs $\tilde{\boldsymbol{x}} \sim P_\theta$ directly from the latent noise $\boldsymbol{z} \sim P_Z$. For discrete Bayesian networks, the generator outputs logits for each categorical variable block, and differentiable relaxed samples are obtained by applying Gumbel--softmax \citep{JangGuPoole:2017gs} to each block before evaluation by the discriminators.  

\paragraph{Restricting the discriminator class.}
In the experiments that consider an interpolative divergence, we impose a Lipschitz-type restriction on each discriminator via spectral normalization to every dense layer \citep{MiyatoKoyamaYoshida18}. This yields the class of constrained discriminators used in the numerical studies and is particularly important in the continuous KL-based experiments, where it prevents degenerate discriminator behavior. 

\paragraph{Validation metrics.}
Across the experiments, we monitor the variational objective together with global, local, and case-specific validation diagnostics. Global distributional fit is assessed using energy distance \citep{SzekelyRizzo:2013es}. For the discrete Bayesian network experiments, where the target conditional probability tables (CPTs) are known, we additionally report the average log-likelihood of generated samples under the ground-truth Bayesian network. Local fit is assessed using the area under the receiver operating characteristic curve (AUC) for each discriminator, while further structure-specific diagnostics are introduced in the relevant case studies. 

For integrable random vectors $\boldsymbol{x}, \boldsymbol{y} \in\mathbb{R}^d$, the energy distance is
$$
\mathcal{E}(\boldsymbol{x}, \boldsymbol{y})=2E |\boldsymbol{x}-\boldsymbol{y}|_d - E|\boldsymbol{x}-\boldsymbol{x}'|_d - E|\boldsymbol{y}-\boldsymbol{y}'|_d\,,
$$
where $\boldsymbol{x}'$ and $\boldsymbol{y}'$ are independent copies. The energy distance is a natural, rotation-invariant extension of the Cramér--von Mises--Smirnov distance in higher dimensions, which compares between-group and within-group variability. We estimate $\mathcal{E}$ between held-out real data and generated samples; for the discrete Bayesian network experiments, this is computed on the corresponding dummy-encoded sample vectors.

For the discrete Bayesian network experiments, the average log-likelihood evaluates how well generated samples agree with the ground-truth Bayesian network. Concretely, we sample from the generator, decode the categorical outputs, and evaluate the resulting configurations under the known joint probability distribution defined by the CPTs. Higher values indicate better agreement with the target distribution. Since the target Bayesian network is known in these experiments, this should be interpreted as an oracle diagnostic rather than as a likelihood for the generator itself.

For each discriminator, we also compute the AUC on held-out real versus generated samples restricted to the corresponding family. AUC values near $0.5$ indicate that the generator matches the data well in the discriminator's domain, whereas larger values indicate remaining discrepancy.

Finally, each case study uses additional structural diagnostics tailored to the example: parent-coefficient recovery in the structured Gaussian experiment, inferred gravitational constant and initial velocity in the ball-throwing experiment, and node-wise total variation error between generated and target marginals in the discrete Bayesian network experiments.

\subsection{Synthetic Data}

We consider two synthetic case studies for continuous Bayesian networks and assess the success of the GiGANs using domain-grounded diagnostics. The first synthetic experiment uses the structured linear-Gaussian Hasse network from Figure~\ref{fig:exm-structured-gaussian} to isolate the effect of discriminator localization. In addition to the objective-level behavior discussed alongside Theorem~\ref{thm:main2general}, we use this example to assess whether graph-informed training preserves global distributional quality using the energy distance and how the objective functional dynamics depend on the Lipschitz constant. The second synthetic example is the ball-throwing time series, whose explicit conditional structure makes it useful for studying practical training dynamics and stability in a continuous setting.

\subsubsection{Structured linear-Gaussian Hasse network}
\label{sec:experiments:structured-linear-gaussian}

We revisit the structured linear-Gaussian model in \eqref{eq:linear-structured-Gaussian} and Figure~\ref{fig:exm-structured-gaussian} that was introduced as a direct experimental testbed for Theorem~\ref{thm:main2general}. In this example, the generator is from a fixed parametric family and the true conditional structure is known, thereby separating the effect of discriminator localization from generator expressivity. Since the earlier discussion in Section~\ref{sec:main-results} already established the main objective-level comparison, the additional diagnostics below show that localized discriminators aligned with the true Bayesian-network families, which improve the recovery of conditional and marginal distributions, do so while maintaining comparable global fit. We also illustrate how the Lipschitz restriction affects the resulting variational dynamics. Unless otherwise stated, the summaries for this case study are based on $13$ paired runs per model; the Lipschitz constant study in Figure~\ref{fig:exm-structured-gauss-lipschitz} is based on only $3$ runs per model. 

For each random seed, we first instantiate the Hasse Gaussian Bayesian network of the subset lattice of a three-element set using the edge weights $\phi_{ji}$ shown in Figure~\ref{fig:exm-structured-gaussian}. We then generate $32{,}000$ i.i.d.\ observations by ancestral simulation. These are split into training and validation sets in a $90{:}10$ ratio. All discriminator constructions are trained on the same instantiated network and the same train--validation split for a given seed, so that differences can be attributed to the discriminator design rather than to differences in the data-generating instance.

To isolate the role of discriminator localization, we fix the parametric family of the generator across all methods. In every case, the generator is a full-covariance Gaussian model
\begin{equation*}
    P_\theta = \mathcal{N}(\mu_\theta,\Sigma_\theta)\,, \qquad \Sigma_\theta = L_\theta L_\theta^\top,
\end{equation*}
with trainable parameters $\theta = (\mu_\theta, L_\theta)$. Thus, the comparison is not between more and less expressive generators, but between different variational objectives induced by different discriminator constructions. This is consistent with Theorem~\ref{thm:main2general}, where the optimization in \eqref{infbounds_2} is over all candidate probability distributions $P \in \mathcal{P}(\mathcal{X})$ and the graph-informed effect enters through the localized objective rather than through an explicit structural restriction on the generator.

Each discriminator is implemented as a neural network with two hidden layers, each with $16$ units and Leaky ReLU activations, trained with the DV--KL objective \eqref{eq:local-obj-dv}. We report results for variational objectives based on both classical and interpolative divergences (i.e., Lipschitz-restricted). In the latter case, spectral normalization is applied to enforce a controlled Lipschitz constant, so that the discriminator class approximates a Lipschitz-bounded class. Training uses Adam with learning rates $(\ell_\eta,\ell_\theta)=(10^{-3},3\times 10^{-3})$, a batch size of $256$, and a maximum of $60$ epochs, with early stopping triggered when the validation energy distance falls below $0.05$ for five consecutive epochs.

The experiment seeks to compare three discriminator constructions. The first is a graph-agnostic baseline with a single monolithic discriminator acting on the full joint distribution. The second is a graph-informed model with one localized discriminator for each child--parent family $F_i=\{i\}\cup\mathrm{Pa}(i)$. The third is a misaligned multi-discriminator baseline in which each true family is replaced by a size-matched subset that excludes the true parents. This misaligned construction is included to test whether any observed gains come from graph alignment itself rather than simply from distributing capacity across several smaller discriminators. For the graph-informed objectives, we use uniform weights across families, so that the overall objective is the average of the local DV terms.

In Section~\ref{sec:main-results}, we assessed the structural fidelity through recovery of the parent coefficients $\{\phi_{ji}\}$ by regressing generated samples onto the true parent set and averaging the resulting $L_2$ error over non-root nodes. Here we complement that local diagnostic with a global one, namely the energy distance on the full eight-dimensional validation sample in Figure~\ref{fig:exm-structured-gauss:global-discrepancy}. Taken together with Figure~\ref{fig:exm-structured-gauss:local-recovery}, the full picture is that the GiGAN improves parent-coefficient recovery while remaining competitive with the monolithic baseline in global energy distance. The misaligned multi-discriminator achieves energy-distance values of a similar order, but fails to recover the correct conditional structure. This shows that the benefit is not simply the use of several discriminators, but the use of discriminators aligned with the correct families.

Figure~\ref{fig:exm-structured-gauss-lipschitz} illustrates how the evolution of the variational loss on the validation set changes as the Lipschitz constant $L$ is varied for the graph-informed multi-discriminator construction, with the monolithic discriminator at $L=1$ included for reference. The plot shows the mean loss over the validation data, with a min-max envelope based on 3 runs for each model. Since the objective functional itself depends on $L$, this figure is best interpreted as showing differences in training dynamics and surrogate behavior rather than providing a direct comparison of absolute objective values across Lipschitz constants. This study shows that the local surrogate \eqref{eq:local-surrogate} changes meaningfully with the Lipschitz constant $L$. In particular, a smaller $L$ produces a lower-valued and more rapidly decaying surrogate, while a larger $L$ yields a higher-valued objective and slower decay. 

In summary, this case study supports the interpretation of Theorem~\ref{thm:main2general} in a controlled setting. The GiGAN performs best on recovering the correct conditional mechanisms, while still maintaining competitive global distributional quality as compared to the graph-agnostic GAN. Moreover, the Lipschitz constant study shows that varying $L$ alters both the objective's magnitude and its training dynamics, suggesting that $L$ can serve as a practical tuning parameter to balance regularization, stability, and sensitivity in graph-informed adversarial training.

\begin{figure}[htb]
\centering
\begin{subfigure}[t]{0.49\textwidth}
  \centering
  \includegraphics[width=\textwidth]{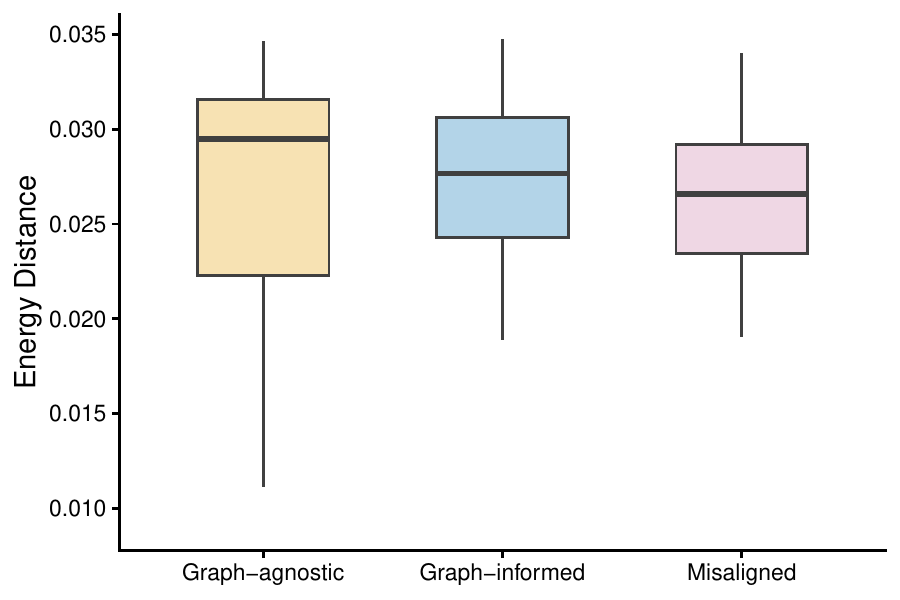}
  \caption{Global discrepancy}
  \label{fig:exm-structured-gauss:global-discrepancy}
\end{subfigure}\hfill
\begin{subfigure}[t]{0.5\textwidth}
  \centering
  \includegraphics[width=\textwidth]{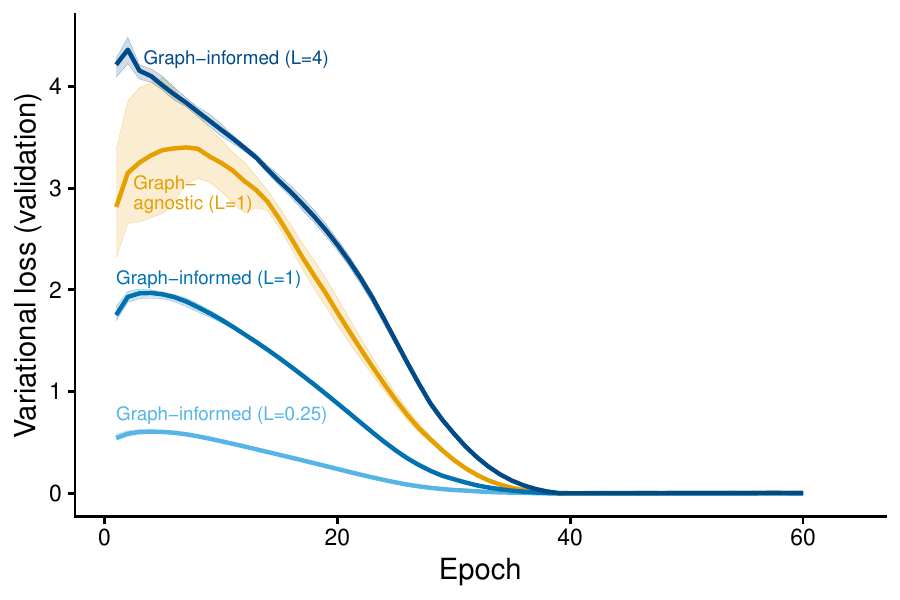}
  \caption{Effect of the Lipschitz constant}
  \label{fig:exm-structured-gauss-lipschitz}
\end{subfigure}

\caption{Additional diagnostics for the structured linear-Gaussian Hasse network, complementing the objective-level comparison in Figure~\ref{fig:exm-structured-gauss}. (a) Energy distance (lower is better) provides a global measure of distributional discrepancy across the three discriminator constructs under the Lipschitz-restricted interpolative divergence (cf.~parent-coefficient recovery in Figure~\ref{fig:exm-structured-gauss:local-recovery}); boxplots summarize $13$ paired runs per model. (b) Variational loss for the GiGAN trained using an interpolative divergence at several Lipschitz constants $L$, with the graph-agnostic GAN at $L=1$ for reference; mean with a min-max envelope based on $3$ runs per model.}
\label{fig:exm-structured-gauss:extra}
\end{figure}

\subsubsection{Ball-throwing trajectory}
\label{sec:experiments:ball-throwing}

We use the synthetic ball-throwing trajectory to study graph-informed adversarial training in a continuous setting where the main issue is practical trainability rather than objective-level certification as in the previous experiment. In this example, the underlying dependence structure is explicit, and the generated trajectories admit physically interpretable diagnostics through the gravitational constant and initial velocity recovered from fitted trajectories. The experiment therefore serves two purposes: it shows how the Lipschitz restriction naturally stabilizes the training, and it tests whether replacing a graph-agnostic monolithic discriminator by graph-informed localized discriminators yields smoother and more diagnostically useful training dynamics. Unless otherwise stated, the summaries in Figure~\ref{fig:exm-ball-training} are based on $20$ independent runs per model for the interpolative setting and $10$ independent runs per model for the classical setting.

The target probability distribution $Q$ is the joint distribution of a discretized collection of ball trajectories with random initial velocities. For a ball with unit mass, the gravitational acceleration $\gbar$, and an initial velocity $v_0$, the trajectory is 
\begin{equation}
    y_t = v_0 t - \frac{1}{2} \gbar  t^2\,.
    \label{eq:ball-trajectory}
\end{equation}
For each independent $v_0 \sim N(\mu_v, \sigma_v^2)$, we observe the time series 
\begin{equation*}
    \boldsymbol{y} = (y_{t_0}, y_{t_1}, \dots, y_{t_m})\,, 
    \qquad  t_j = j/m, \quad j = 0, ..., m\,,
\end{equation*}
at $m+1 = 15$ uniformly spaced time points in $[0,1]$. In the experiments below, we use $\mu_v=4$ and $\sigma_v=3$. Figure~\ref{fig:exm-ball} shows fifty representative trajectories from the target distribution together with generated trajectories from graph-informed and graph-agnostic models.

As \eqref{eq:ball-trajectory} is a second-order equation in time, the resulting joint probability distribution has a natural local dependence structure: 
conditional on the two previous positions, the next position is determined up to the randomness induced by $v_0$.
This motivates the Bayesian network description, 
\begin{equation}
\label{eq:bn-ball-trajectory}
  y_{t_{j-1}} \rightarrow y_{t_j} \leftarrow y_{t_{j-2}}\,, \quad j = 2, \dots, m\,,
\end{equation}
together with the boundary relation $y_{t_0} \rightarrow y_{t_1}$. In the graph-informed model, the localized discriminator families are chosen to reflect this structure. In particular, the families correspond to triples of the form $(y_{t_j}, y_{t_{j-1}}, y_{t_{j-2}})$, with the obvious truncation near the initial time.

\begin{figure}[htb]
\centering
\begin{subfigure}[t]{0.49\textwidth}
  \centering
  \includegraphics[width=\textwidth]{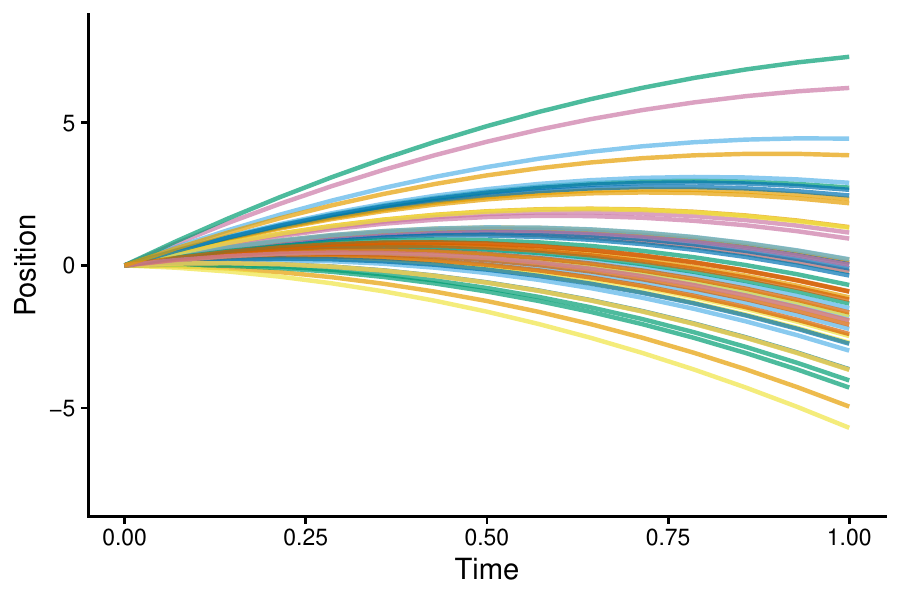}
  \caption{Synthetic data}
  \label{fig:exm-ball:data}
\end{subfigure}\\
\begin{subfigure}[t]{0.49\textwidth}
  \centering
  \includegraphics[width=\textwidth]{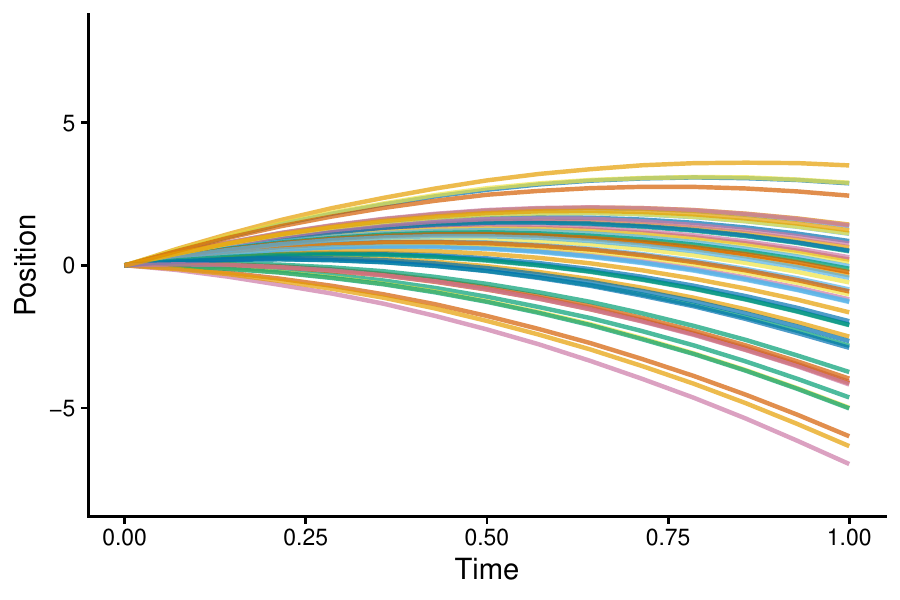}
  \caption{GiGAN}
  \label{fig:exm-ball:m-1}
\end{subfigure}\hfill
\begin{subfigure}[t]{0.49\textwidth}
  \centering
  \includegraphics[width=\textwidth]{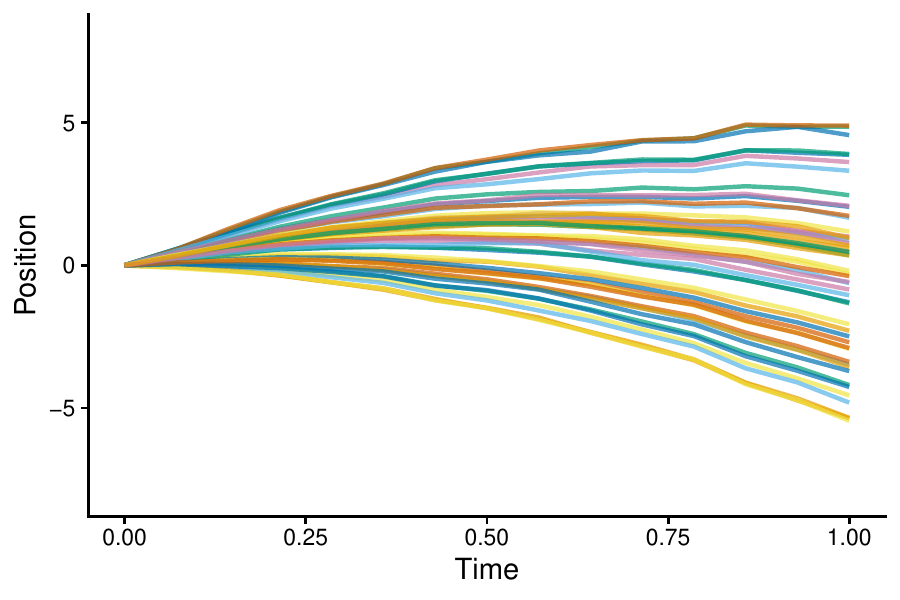}
  \caption{Graph-agnostic GAN}
  \label{fig:exm-ball:m-none}
\end{subfigure}

\caption{Fifty ball-throwing trajectory time series from various distributions. (a) Samples from the target distribution with $v_0\sim N(\mu_v=4,\sigma_v^2=3^2)$. (b) Fifty generated trajectories from a GiGAN with $14$ localized discriminators trained using an interpolative divergence with a Lipschitz constraint. (c) Fifty generated trajectories from a graph-agnostic GAN with a monolithic discriminator trained using an interpolative divergence with a Lipschitz constraint.}
\label{fig:exm-ball}
\end{figure}

We compare a graph-agnostic model, with a single monolithic discriminator acting on the full trajectory, against a graph-informed model, with $14$ localized discriminators aligned with the  Bayesian-network child--parent families. The discriminators are fully connected networks with two hidden layers of width $32$ and Leaky ReLU activations. Both models use the same generator architecture (a fully-connected feed-forward network with latent dimension $64$, two hidden layers of width $128$, and swish activations, %
outputting a length-$15$ trajectory) and are trained with the Donsker--Varadhan form of the KL objective in \eqref{eq:local-obj-dv} with minibatches of size $256$ using Adam with one discriminator update and one generator update per minibatch with learning rates
$(\ell_\eta,\ell_\theta)=(5\times 10^{-4},10^{-3})$. To assess the role of interpolative regularization, we also compare training with and without spectral normalization. In the constrained setting, spectral normalization is applied to each discriminator layer so that the discriminator class approximates a Lipschitz-bounded class; in the unconstrained setting, the same discriminator architecture is trained directly without this restriction.

The generated trajectories also admit interpretable physics-based diagnostics.  By fitting the quadratic model in \eqref{eq:ball-trajectory} to generated samples via least squares, we obtain estimates  $\hat{v}_0$ and $\hat{\gbar}$ from the fitted linear and quadratic coefficients. We track the empirical mean and spread of these estimates across generated samples, as a good model should reproduce both the distribution of initial velocities and the true gravitational constant,  not merely match the trajectories in a generic distributional sense. These physics-based diagnostics provide an interpretable validation signal for model selection and early stopping. 

\begin{figure}[htb]
\centering
\begin{subfigure}[t]{0.49\textwidth}
  \centering
  \includegraphics[width=\textwidth]{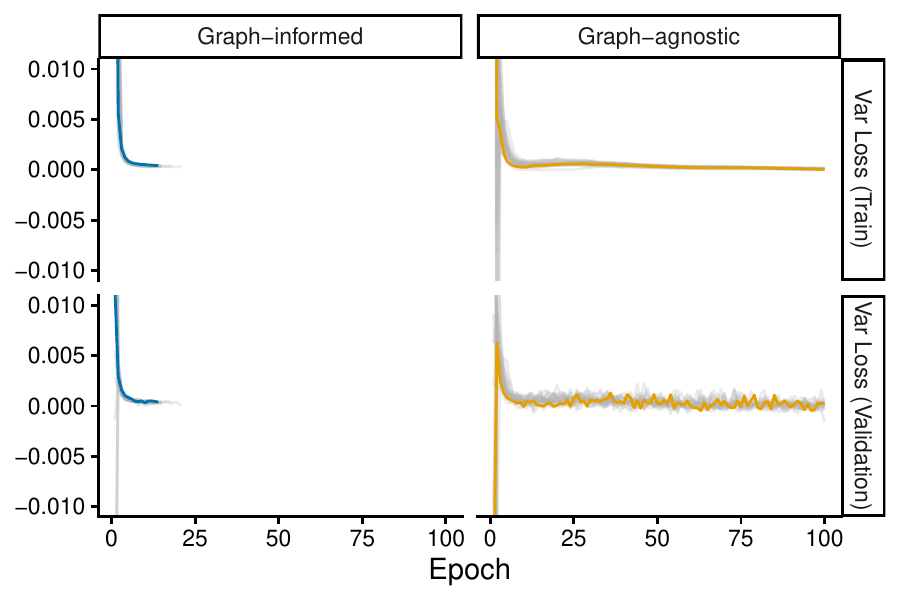}
  \caption{Variational loss (interpolative)}
  \label{fig:exm-ball:var-loss}
\end{subfigure} \hfill
\begin{subfigure}[t]{0.49\textwidth}
  \centering
  \includegraphics[width=\textwidth]{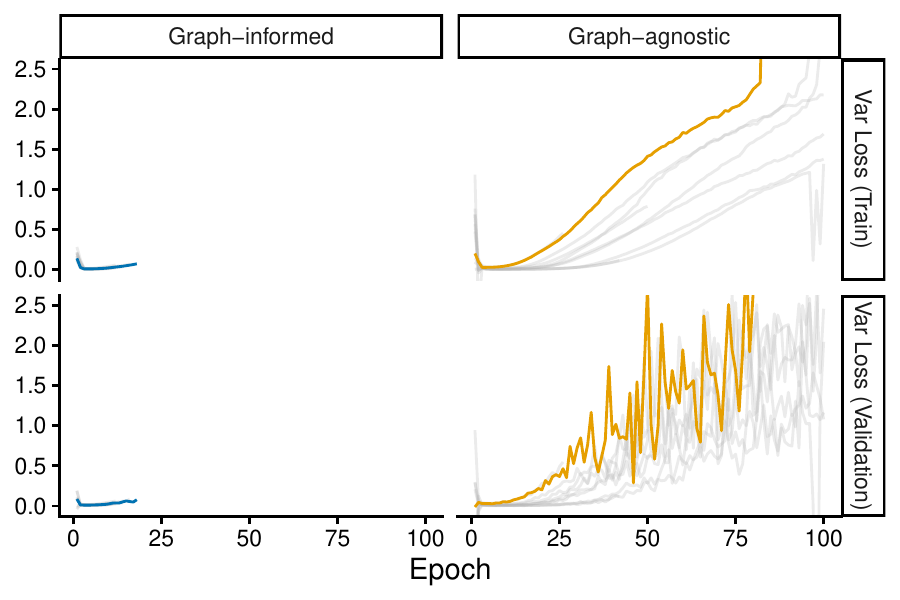}
  \caption{Variational loss (classical)}
  \label{fig:exm-ball:var-loss-specFalse}
\end{subfigure} \\
\medskip
\begin{subfigure}[t]{0.49\textwidth}
  \centering
  \includegraphics[width=\textwidth]{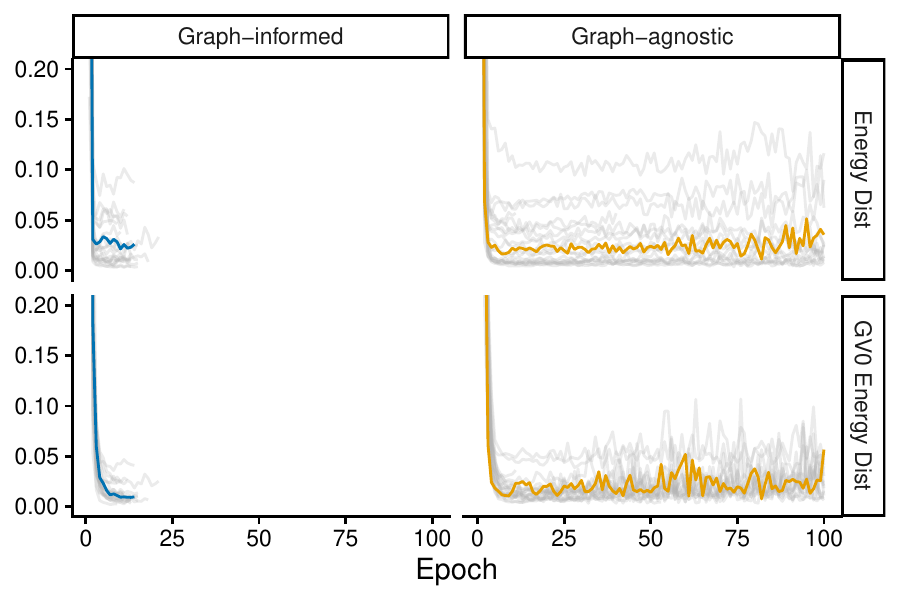}
  \caption{Energy distance (interpolative)}
  \label{fig:exm-ball:energy-dist}
\end{subfigure} \hfill
\begin{subfigure}[t]{0.49\textwidth}
  \centering
  \includegraphics[width=\textwidth]{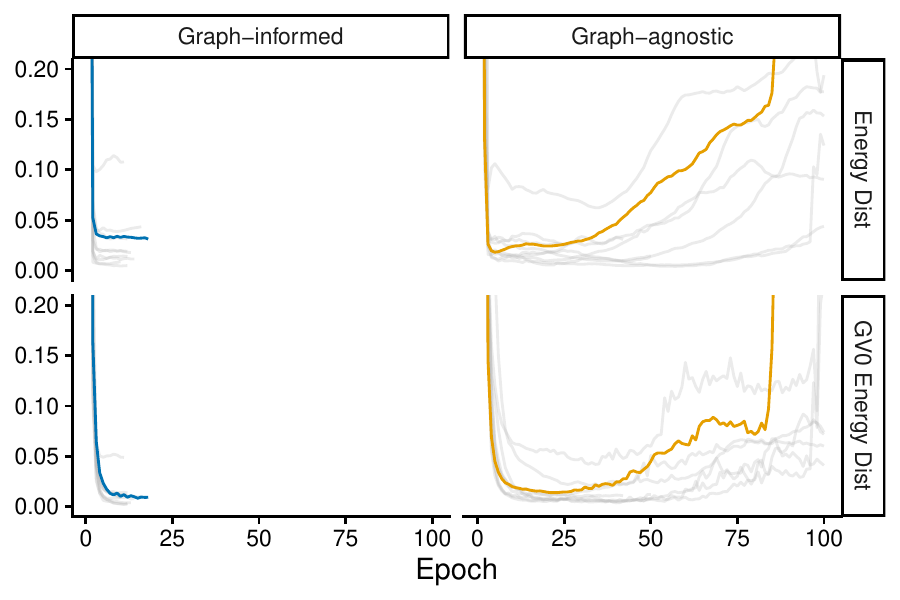}
  \caption{Energy distance (classical)}
  \label{fig:exm-ball:energy-dist-specFalse}
\end{subfigure}
\caption{Training diagnostics for the ball-trajectory experiment from a representative run; the GiGAN with $14$ localized discriminators aligned to child--parent families in the  Bayesian network is compared to the graph-agnostic GAN with a monolithic discriminator. (a) and (b) show the evolution of the KL variational loss \eqref{eq:local-obj-dv} on the training and validation sets, for an interpolative ($20$ runs in total) and classical divergence  ($10$ runs in total), i.e.,  with and without Lipschitz constraints enforced via spectral normalization. (c) and (d) show the corresponding energy-distance diagnostics, both for the full trajectory distribution and for the recovered physics parameters $(\hat{\gbar},\hat{v}_0)$, again for an interpolative and classical divergence.}
\label{fig:exm-ball-training}
\end{figure}

Figure~\ref{fig:exm-ball-training} summarizes the main training diagnostics. Panels~(a) and~(b) show the evolution of the variational objective with and without the Lipschitz restriction. When using an interpolative divergence, both graph-informed and graph-agnostic training remain numerically stable and produce interpretable loss curves. Without this restriction, the DV--KL objective becomes much less well behaved, reflecting the instability of the log-moment term in the unconstrained setting. Panels~(c) and~(d) show the corresponding energy-distance diagnostics. The main qualitative effect is that the interpolative restriction stabilizes training, while graph-informed localization makes the validation curves less oscillatory and therefore easier to use for model monitoring and stopping.

This stabilization is especially visible when comparing the graph-informed model with the graph-agnostic model. Although the localized model updates multiple discriminators per minibatch, its training curves are smoother, and the inferred physics is more stable across epochs. %
In this example, the benefit of graph-informed localization is a more reliable optimization process and a clearer diagnostic picture under a comparable wall-clock budget.

Overall, the ball-throwing experiment complements the Hasse-network study by highlighting a different advantage of graph-informed training. In the Hasse example, the main gain was improved structural recovery under a known Bayesian-network factorization. Here, the main illustrated gain is improved practical trainability in a structured continuous model: the interpolative restriction prevents DV--KL degeneracy, and graph-informed localized discriminators yield smoother and more interpretable training dynamics than a graph-agnostic monolithic discriminator.

\subsection{Real-world Bayesian networks}
\label{sec:experiments:real-bn}

We now turn to two real-world Bayesian networks in order to study graph-informed adversarial training beyond the synthetic continuous setting. The first example uses the \textsc{Child} network from \cite{spiegelhalter1992learning}, a medium-sized Bayesian network consisting of $20$ categorical variables and $25$ directed edges ($230$ parameters) that describe clinical factors associated with birth asphyxia (Figure~\ref{fig:exm-child-bn}). The second example uses the \textsc{Earthquake} network from \cite{pearl1988probabilistic} (see also \cite{KorbNicholson:2010bi}), a smaller Bayesian network consisting of $5$ nodes and $4$ directed edges ($10$ parameters). In both cases, the network structure and CPTs are assumed to be known. These experiments are intended to test whether the GiGAN continues to improve structural fidelity in discrete settings, and whether the resulting multi-discriminator training remains stable across repeated runs.

\begin{figure}[htb]
\centering
\begin{subfigure}[t]{0.74\textwidth}
  \centering
  \includegraphics[width=1\textwidth]{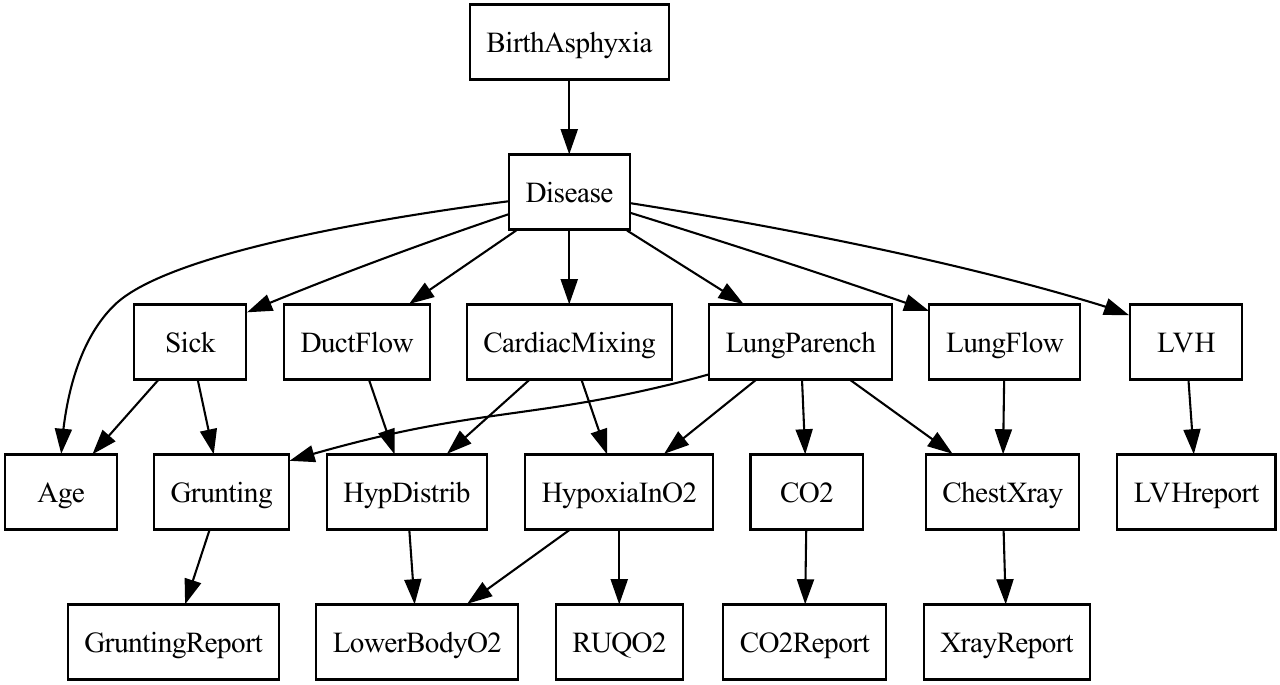}
  \caption{\textsc{Child}}
  \label{fig:exm-child-bn}
\end{subfigure}\hfill
\begin{subfigure}[t]{0.20\textwidth}
  \centering
  \includegraphics[width=\textwidth]{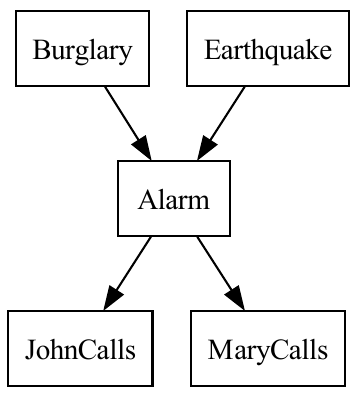}
  \caption{\textsc{Earthquake}}
  \label{fig:exm-earthquake-bn}
\end{subfigure}
\caption{The two real-world discrete Bayesian networks used in the experiments. (a) \textsc{Child}, a medium-sized, discrete Bayesian network with $20$ nodes and $25$ arcs ($230$ parameters) describing a possible diagnostic tool for birth asphyxia for newborn babies with congenital heart disease, from \cite{spiegelhalter1992learning}. (b) \textsc{Earthquake}, a small-sized, discrete Bayesian network consisting of 5 nodes and 4 arcs (10 parameters), from \cite{pearl1988probabilistic,KorbNicholson:2010bi}.}
\label{fig:exm-real-bns}
\end{figure}

\subsubsection{Localization for discrete Bayesian networks}
\label{sec:experiments:discrete-localization}

Since \textsc{Child} and \textsc{Earthquake} are fully discrete, we employ a differentiable relaxation to enable gradient-based training. The generator outputs a single dummy-encoded vector of logits, partitioned into blocks corresponding to the categorical variables. We obtain differentiable samples by applying the Gumbel--softmax relaxation of \cite{JangGuPoole:2017gs} independently to each block. Gumbel noise is added to the logits, and a temperature-controlled softmax produces an approximate dummy encoding per variable. These relaxed samples are passed to the discriminators during training, while hard-decoded samples are used for diagnostics.

All models in this section are trained with an interpolative JS-Wasserstein-type objective. That is, we consider the graph-informed objective in \eqref{eq:sum-local-obj}, using the $f$-GAN form in \eqref{eq:local-obj-f-gan} with the JS divergence and layer-wise spectral normalization in the discriminator to enforce a Lipschitz constraint.

The diagnostics reflect that the target Bayesian networks are fully known. We monitor the variational objective during training, the energy distance on held-out validation samples, and the average log-likelihood of generated samples under the known conditional probability tables. Since this log-likelihood is evaluated under the ground-truth Bayesian network rather than under the generator itself, it should be interpreted as a measure of agreement with the target probability distribution. To assess local fit, we report discriminator AUC values on held-out real versus generated samples. We also compute node-wise total variation error between generated marginals and the corresponding target marginals.

\subsubsection{\texorpdfstring{\textsc{Child}}{CHILD}}
\label{sec:experiments:child}

The \textsc{Child} experiment assesses the GiGAN on a medium-sized discrete Bayesian network in the theory-aligned child--parent setting and examines how the size of the localized discriminator families affects the trade-off between local and global fit of model predictions. We report results over $15$ independent runs. 

For each \textsc{Child} run, we draw $64{,}000$ i.i.d.\ cases from the exact Bayesian network by ancestral sampling, dummy-encode the resulting categorical observations into a $60$-dimensional vector, and split the data $90{:}10$ into training and validation sets. The generator is a fully connected network with latent dimension $32$, hidden widths $(64,64,64,64)$, swish activations, and a linear output layer producing logits for all dummy variables. The discriminators are fully connected networks with hidden widths $(16,8,4)$ and Leaky ReLU activations. Training uses Adam with learning rates $(\ell_\eta, \ell_\theta) = (5\times 10^{-4},\,5\times 10^{-5})$, batch size $512$, and at most $150$ epochs, with early stopping based on the validation average log-likelihood under the ground-truth CPTs (patience $20$, minimum improvement $10^{-3}$), after which the best checkpoint is restored. The graph-agnostic baseline uses a single discriminator on the full $60$-dimensional dummy vector. The graph-informed model uses $20$ localized discriminators aligned with the network families.

For \textsc{Child}, we write $M$ for the graph-agnostic GAN model with a monolithic discriminator, and $M_a$ for the GiGAN model with localized discriminators acting on the families of ancestor depth $a$,
$$F_i^{(a)} = \{i\} \cup \mathrm{An}_{\leq a}(i)\,,$$
where $\mathrm{An}_{\leq a}(i)$ denotes the ancestors of $i$ within graph distance at most $a$. In particular, $a=1$ corresponds to the typical child--parent family $F^{(1)}_i = \{i\}\cup \mathrm{Pa}(i)$. We compare the graph-agnostic model $M$ with the graph-informed models $M_1$ and $M_3$.

The model $M_1$ aligns with the graph-informed child--parent families covered by our main framework. By contrast, the model $M_3$ is included as an exploratory extension motivated by Remark~\ref{rem:tighter_bound}. This larger family localization is not presented as a direct consequence of Theorem~\ref{thm:main2general}, but rather as a way to examine how broader family-aligned neighborhoods affect the discriminator's inductive bias.  

Figure~\ref{fig:exm-real-bn-training:child} first shows that graph-informed training produces smoother variational-loss curves than the graph-agnostic baseline. Empirically, $M_3$ tends to occupy an intermediate regime between the more local graph-informed model $M_1$ and the monolithic graph-agnostic model $M$. Both $M_1$ and $M_3$ exhibit less oscillatory training and validation behavior than $M$, which makes convergence easier to diagnose. This mirrors the stabilization effects already observed in the continuous experiments, but now in a discrete setting using the interpolative JS objective functional.

The global validation metrics in Figure~\ref{fig:exm-real-bn-validation:child-global} show that localization changes the balance between local and global fit. The graph-informed models typically achieve higher average log-likelihood under the ground-truth Bayesian network than the graph-agnostic baseline, indicating better agreement with the known conditional structure. The energy distance, however, does not improve uniformly: model $M_1$ can yield slightly worse global energy distance than the monolithic baseline, while the model $M_3$ typically lies between the two. This is consistent with the view that larger neighborhoods recover more of the global dependence structure, whereas smaller families place greater emphasis on local fidelity.

This local advantage is reflected more clearly in the discriminator AUC and total-variation diagnostics. In Figure~\ref{fig:exm-real-bn-validation:child-auc}, the AUC values for the localized discriminators in $M_1$ and $M_3$ move more consistently toward $0.5$ on their respective families than does the monolithic discriminator in $M$ on the full joint space. For reference, the displayed critics correspond to the nodes BirthAsphyxia ($D_0$), HypoxiaInO2 ($D_2$), XrayReport ($D_{10}$), CardiacMixing ($D_{16}$), and Sick ($D_{19}$); see Figure~\ref{fig:exm-child-bn}. The node-wise total-variation errors in Figure~\ref{fig:exm-real-bn-tv:child-by-node-by-depth} show the same pattern. Although all fitted models achieve relatively small marginal errors, the graph-informed models reduce these errors systematically, with $M_1$ attaining the lowest mean node-wise total variation across runs in Figure~\ref{fig:exm-real-bn-tv:child-node-wise-mean}.

The \textsc{Child} experiment shows that on a medium-sized discrete Bayesian network, graph-informed localization improves structural fidelity, as directly reflected in node-wise total variation, local AUC, and log-likelihood. Broader neighborhoods partly recover global fit, but the main advantage of the GiGAN in this setting is better alignment with the conditional structure encoded by the Bayesian network.

\begin{figure}[htb]
\centering
\begin{subfigure}[t]{0.49\textwidth}
  \centering
  \includegraphics[width=\textwidth]{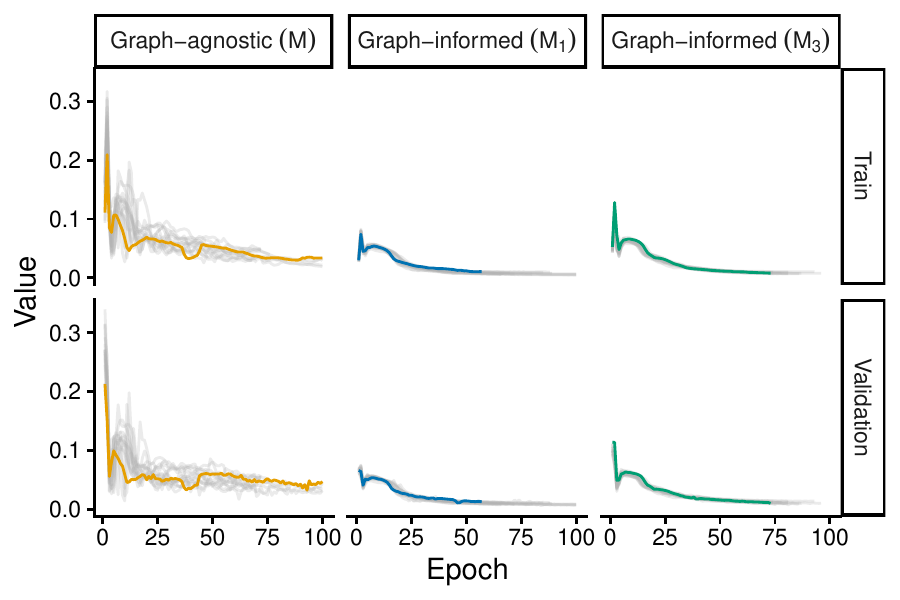}
  \caption{\textsc{Child} variational loss}
  \label{fig:exm-real-bn-training:child}
\end{subfigure}\hfill
\begin{subfigure}[t]{0.49\textwidth}
  \centering
  \includegraphics[width=\textwidth]{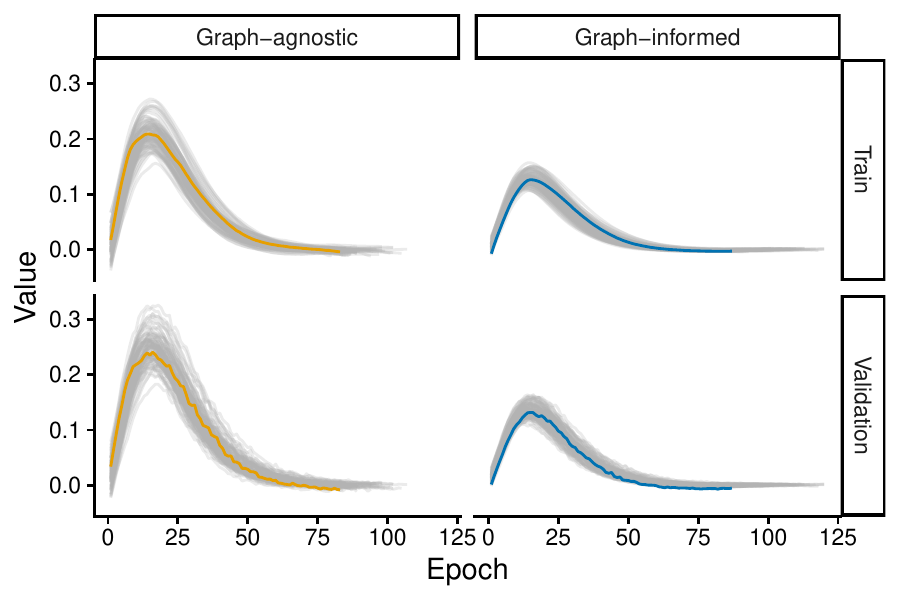}
  \caption{\textsc{Earthquake} variational loss}
  \label{fig:exm-real-bn-training:earthquake}
\end{subfigure}
  \caption{Training and validation variational loss (interpolative JS-Wasserstein) for the real-world discrete Bayesian network experiments. (a) Compares the graph-agnostic GAN model $M$ with the GiGAN models $M_1$ (child--parent families) and $M_3$ (ancestor depth-$3$ families) on \textsc{Child} over $15$ runs. (b) Compares the graph-agnostic GAN and the GiGAN on \textsc{Earthquake} over $85$ runs.}
  \label{fig:exm-real-bn-training}
\end{figure}

\begin{figure}[htb]
\centering
\begin{subfigure}[t]{0.49\textwidth}
  \centering
  \includegraphics[width=\textwidth]{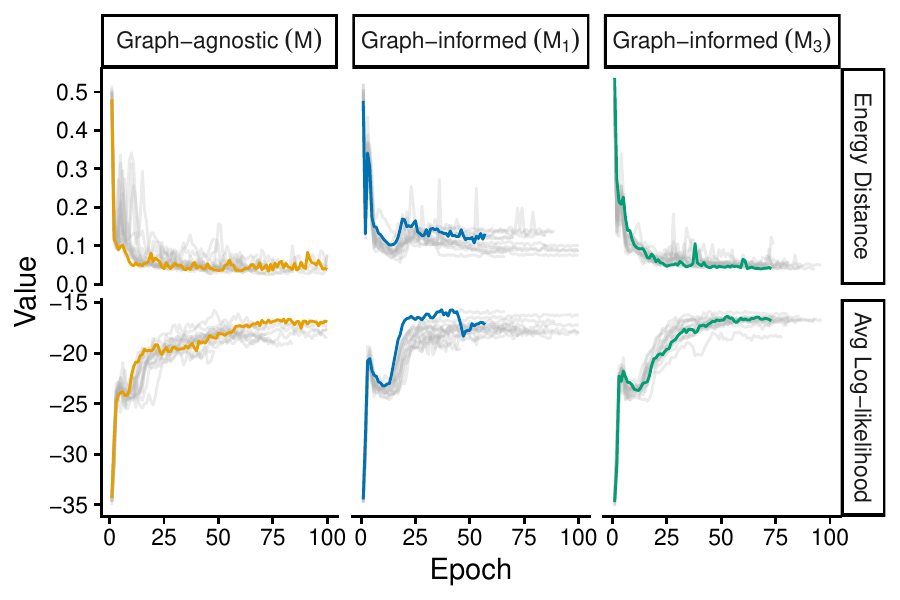}
  \caption{\textsc{Child} global validation}
  \label{fig:exm-real-bn-validation:child-global}
\end{subfigure}\hfill
\begin{subfigure}[t]{0.49\textwidth}
  \centering
  \includegraphics[width=\textwidth]{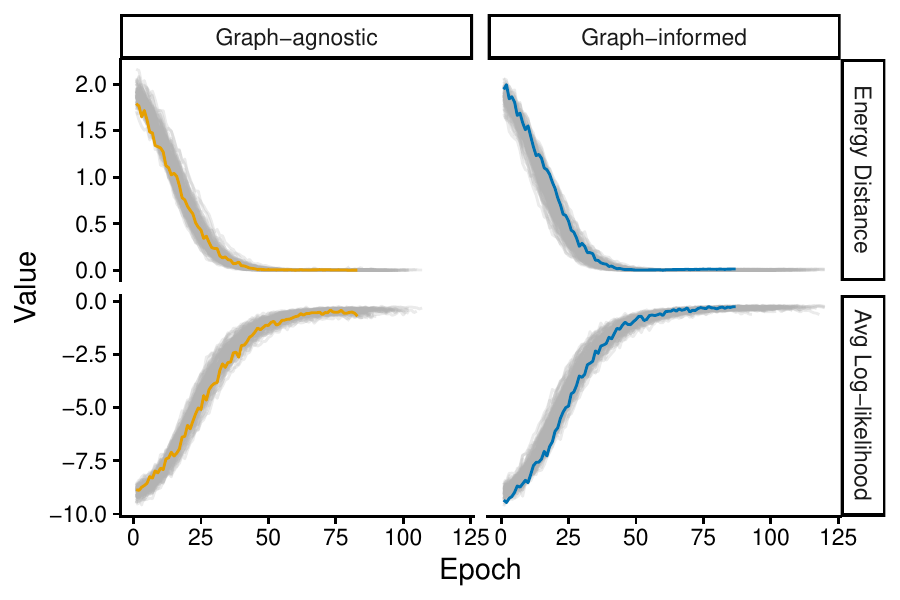}
  \caption{\textsc{Earthquake} global validation}
  \label{fig:exm-real-bn-validation:earthquake-global}
\end{subfigure} \medskip \\
\begin{subfigure}[t]{0.49\textwidth}
  \centering
  \includegraphics[width=\textwidth]{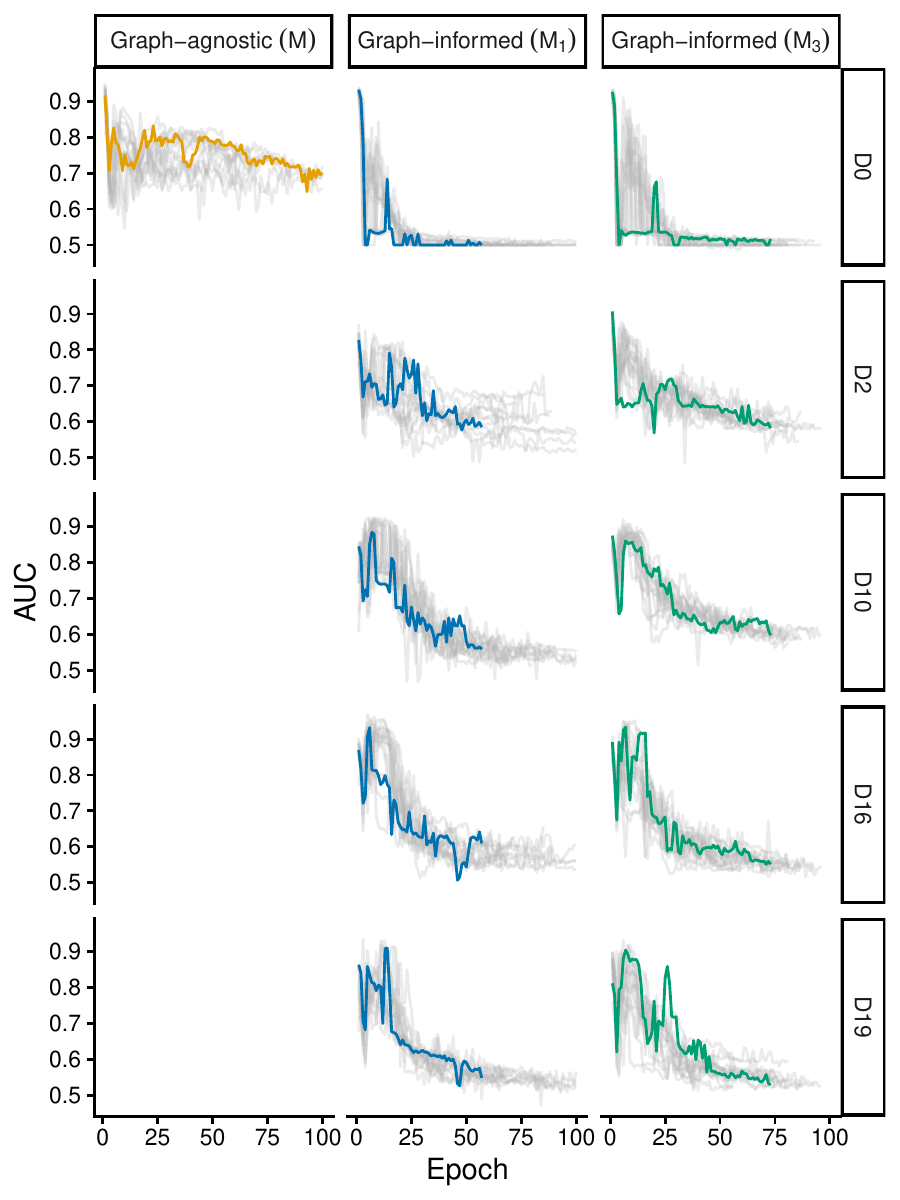}
  \caption{\textsc{Child} AUC}
  \label{fig:exm-real-bn-validation:child-auc}
\end{subfigure}\hfill
\begin{subfigure}[t]{0.49\textwidth}
  \centering
  \includegraphics[width=\textwidth]{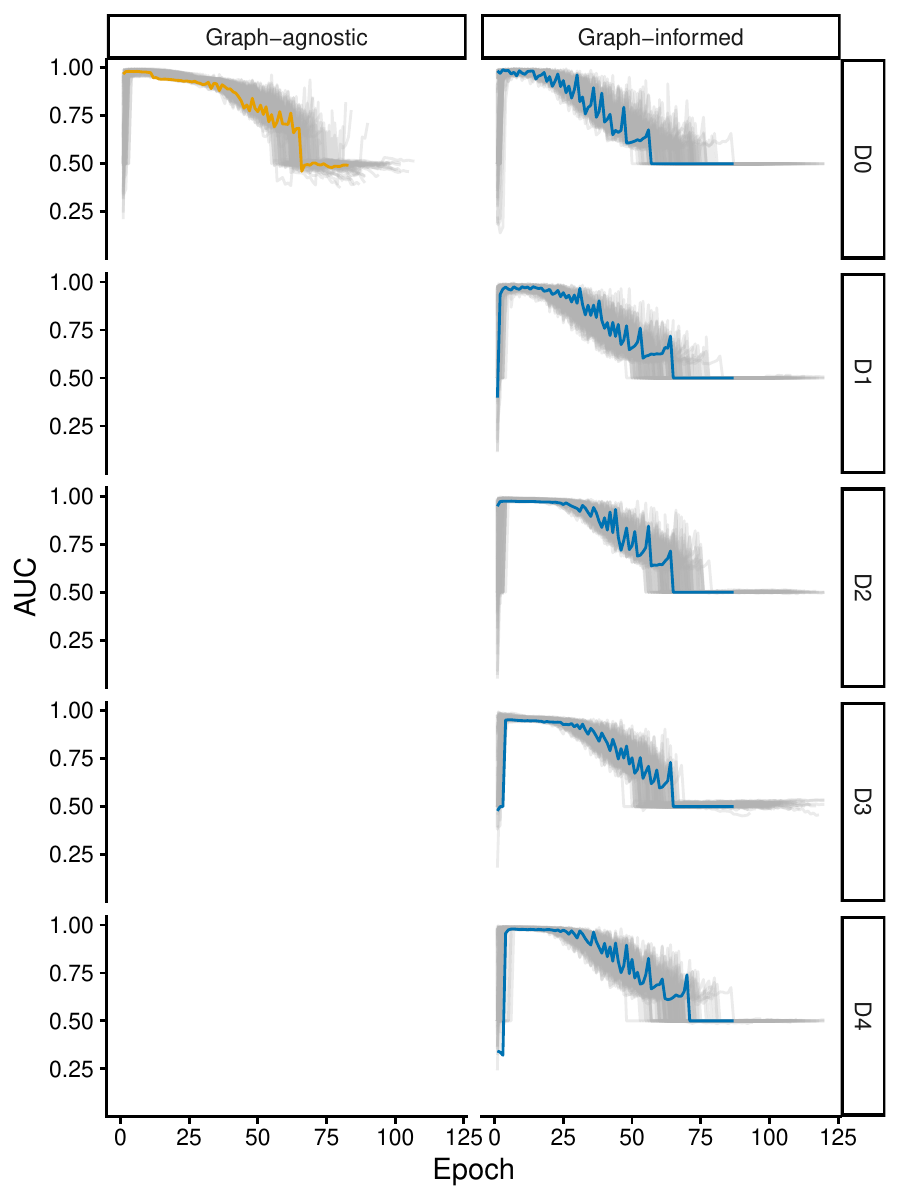}
  \caption{\textsc{Earthquake} AUC}
  \label{fig:exm-real-bn-validation:earthquake-auc}
\end{subfigure}
\caption{Validation diagnostics for the real-world discrete Bayesian network experiments (single run highlighted). The \textsc{Child} experiment compares $M$, $M_1$, and $M_3$, corresponding to graph-agnostic GAN and two GiGANs with child--parent and ancestor depth-3 families over $15$ runs. The \textsc{Earthquake} experiment compares graph-agnostic and GiGANs over $85$ runs. (a) and (b) show global validation metrics, namely energy distance and average log-likelihood. (c) and (d) show representative discriminator AUC values. }
\label{fig:exm-real-bn-validation}
\end{figure}

\begin{figure}[htb]
\centering
\begin{subfigure}[t]{0.49\textwidth}
  \centering
  \includegraphics[width=\textwidth]{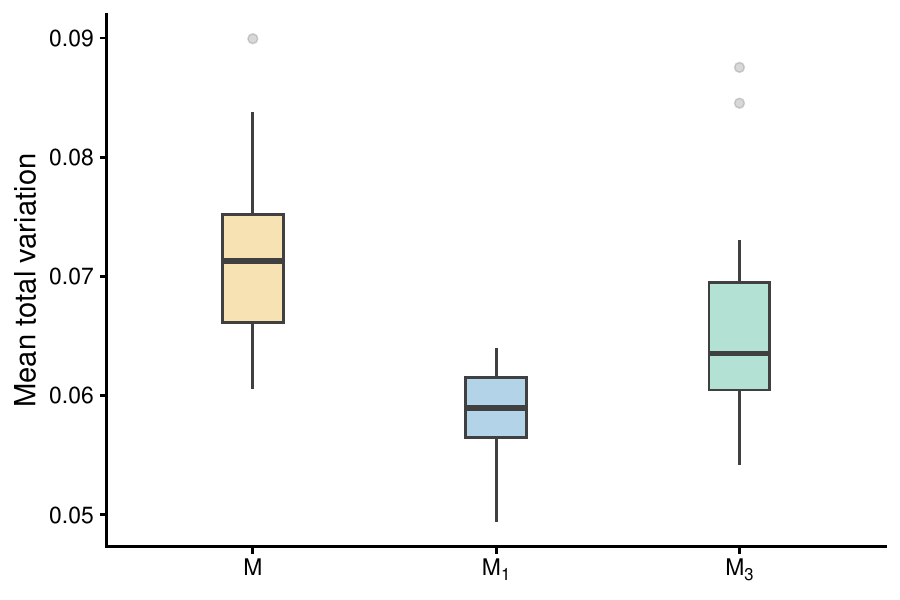}
  \caption{\textsc{Child} mean total variation}
  \label{fig:exm-real-bn-tv:child-node-wise-mean}
\end{subfigure} \hfill
\begin{subfigure}[t]{0.49\textwidth}
  \centering
  \includegraphics[width=\textwidth]{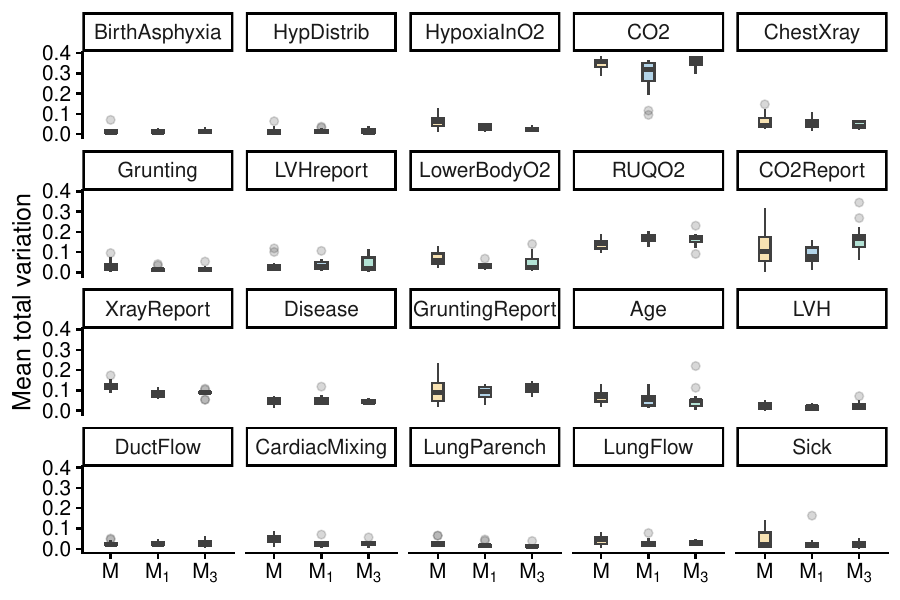}
  \caption{\textsc{Child} mean total variation by node}
  \label{fig:exm-real-bn-tv:child-by-node-by-depth}
\end{subfigure}
\medskip \\
\begin{subfigure}[t]{0.49\textwidth}
  \centering
  \includegraphics[width=\textwidth]{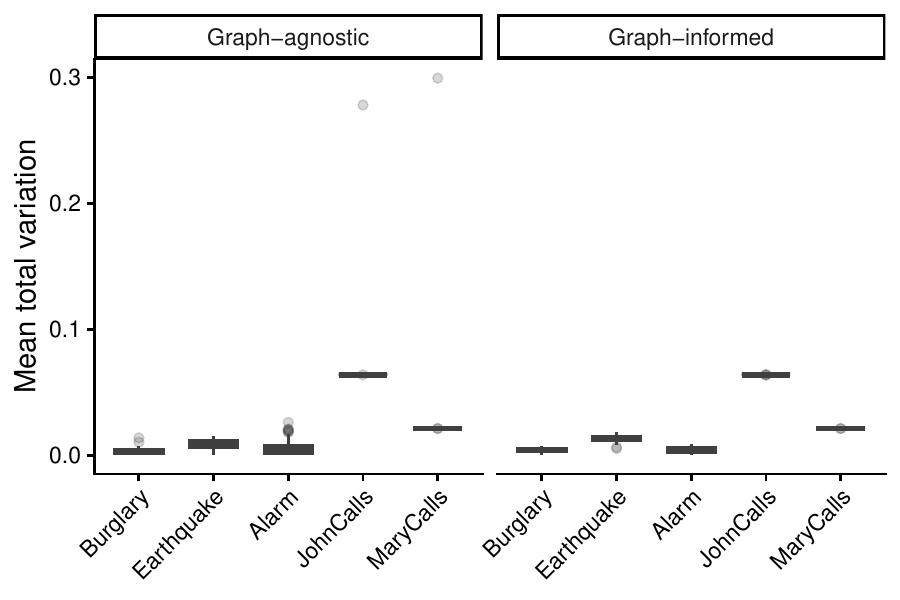}
  \caption{\textsc{Earthquake} mean total variation by node}
  \label{fig:exm-real-bn-tv:earthquake-by-depth-by-node}
\end{subfigure}
\caption{Total variation error between generated marginals and the corresponding target marginals from the known CPT. (a) The mean error by node by depth, across $15$ runs for \textsc{Child} for graph-agnostic model $M$ and graph-informed models $M_1$ and $M_3$. (b) The mean error, averaged over all nodes by depth, across $15$ runs for \textsc{Child}, highlighting variables such as CO2 with high error relative to other nodes. (c) The mean error by depth by node, across $85$ runs for \textsc{Earthquake}; note the outliers in the graph-agnostic predictions.}
\label{fig:exm-real-bn-tv}
\end{figure}

\subsubsection{\texorpdfstring{\textsc{Earthquake}}{EARTHQUAKE}}

The \textsc{Earthquake} experiment serves a different purpose from \textsc{Child}. Because the network is small, the graph-agnostic GAN with a monolithic discriminator is already operating in a relatively low-dimensional setting, so the main question is not whether localization dramatically improves the average metric values, but whether it improves reliability across repeated runs. Here we compare the graph-agnostic model with the graph-informed model using the child--parent families and report results over $85$ independent runs.

For each \textsc{Earthquake} run, we draw $8{,}000$ i.i.d.~cases from the exact Bayesian network by ancestral sampling, dummy-encode them into a $10$-dimensional vector, and use a $90{:}10$ train--validation split. The generator is a fully connected network with latent dimension $16$, two hidden layers with widths $32$, swish activations, and linear output logits. The discriminators are fully connected networks with two hidden layers also with widths $32$ and Leaky ReLU activations. Training uses Adam with learning rates $(\ell_\eta, \ell_\theta) = (10^{-4},\,5\times 10^{-5})$, batch size $256$, and at most $120$ epochs, with early stopping based on the validation average log-likelihood under the ground-truth CPTs, after which the best checkpoint is restored. The graph-agnostic baseline uses one discriminator on the full $10$-dimensional dummy vector, whereas the graph-informed model uses $5$ localized discriminators aligned with the child--parent families; the corresponding family input widths range from $2$ to $6$ dummy coordinates.

Figure~\ref{fig:exm-real-bn-training:earthquake} again shows smoother training dynamics for the GiGAN than for the graph-agnostic GAN. The difference is less dramatic than for \textsc{Child}, but the graph-informed model still exhibits less erratic variation in the training and validation objectives. In the global validation metrics shown in Figure~\ref{fig:exm-real-bn-validation:earthquake-global}, the two models are broadly comparable in terms of their central tendency. This is consistent with the fact that the full joint distribution is small enough that the monolithic discriminator is not severely disadvantaged.

The main difference between the graph-agnostic and graph-informed models appears in robustness. In Figure~\ref{fig:exm-real-bn-validation:earthquake-auc}, the graph-informed discriminator AUC behaves more consistently across runs, while the monolithic baseline exhibits more variable behavior. The same effect is visible in the node-wise total-variation errors in Figure~\ref{fig:exm-real-bn-tv:earthquake-by-depth-by-node}: although the average performance of the two approaches is similar, the graph-agnostic model exhibits several more pronounced outliers. Thus, for \textsc{Earthquake}, the principal gain from graph-informed localization is not uniformly better average accuracy but improved reliability and fewer pathological runs.

This provides a useful boundary case for the broader conclusions of the paper. When the joint space is small, the GiGAN still offers a practical advantage by making training more stable and reducing the frequency of poor outcomes across repeated runs.

\section{Discussion}
\label{sec:discussion}

We have introduced an infimal subadditivity principle for interpolative divergences on Bayesian networks and used it to give a variational justification for graph-informed adversarial learning. In particular, Theorems~\ref{thm:main2general}, \ref{thm:main_IPM}, and \ref{thm:pOTd} show that, when the target distribution factorizes according to a known DAG, a global discrepancy objective functional can be controlled by an average of family-level objective functionals aligned with the graph. This provides a principled basis for replacing a graph-agnostic GAN with a monolithic discriminator by a \emph{GiGAN}, a graph-informed GAN with localized discriminators trained under an interpolative divergence objective, without requiring the optimizer itself to factorize according to the same graph. The same variational framework applies to $(f,\Gamma)$-divergences, $\Gamma$-IPMs, and proximal optimal transport divergences, and applies to natural discriminator classes including bounded measurable and Lipschitz-bounded functions.

The experiments suggest that the practical value of graph-informed adversarial learning is not only computational. Across the synthetic and real-world case studies, localization often improves training stability, reduces variability across runs, and yields better recovery of the underlying conditional structure in the generated distribution. At the same time, the \textsc{Child} experiments with ``extended families'' suggest that the choice of family size, or more generally of the graph-informed neighborhoods, mediates a trade-off between local structural fidelity and global distributional fit. This complements Remark~\ref{rem:tighter_bound}: broader neighborhoods may tighten the surrogate and recover more of the global dependence structure, but they may also reduce some of the advantages of strict localization in terms of simplicity, efficiency, and interpretability. ``Learning big'' with a single monolithic discriminator versus ``learning small'' with localized family-level discriminators should be viewed as a task-dependent modeling choice.

A related point is that localization also improves diagnosability. Because the graph-informed objective decomposes into family-level terms, the resulting training procedure naturally produces local, family-specific diagnostics, such as discriminator AUCs and node-wise errors. As seen in the experiments, these quantities can make it easier to identify where the model is failing to train, rather than only indicating that the full-joint fit remains inadequate. This is potentially valuable in structured applications where understanding which part of the distribution is poorly learned may matter as much as aggregate performance. From this perspective, graph-informed adversarial learning offers not only a computational decomposition of a difficult high-dimensional testing problem, but also a more interpretable diagnostic interface for model development which may be especially relevant in structured and potentially high-consequence applications.

Finally, our theory distinguishes between structured objectives and structured generators. The surrogate controls a global objective functional without requiring the optimizer itself to be explicitly restricted to $\mathcal{P}^G$, and this separation is one of the main conceptual advantages of the infimal formulation. It would therefore be natural to study how graph-informed discriminator design interacts with generators that are themselves structurally constrained, and whether combining both sides can yield sharper guarantees, improved sample efficiency, or better-calibrated uncertainty. More broadly, it would be interesting to better understand how the choice of neighborhoods, discrepancy class, and discriminator regularity should be adapted to the scientific or decision-making task at hand. In this sense, the present work should be viewed less as prescribing a single graph-informed architecture and more as establishing a general variational framework for learning with known probabilistic structure.

\bibliographystyle{plainnat}
\bibliography{referencesBH,ref}

\appendix

\section{Background and properties of measures of discrepancy}
\label{appendix:extra-definitions}

This appendix collects background definitions and auxiliary facts used in the main text. In particular, we record the notions of determining, admissible, and strictly admissible function classes used in Section~\ref{sec:var-gans}.

\begin{definition}[$\mathcal P(\mathcal X)$- determining set] A set $\Gamma \subset C_b(\mathcal{X})$ will be called $\mathcal P(\mathcal X)$-determining set if for all $Q, P \in \mathcal{P}(\mathcal{X})$, 
\begin{equation*}
  \int \varphi \, Q(dx) = \int \varphi \, P(dx) \quad \forall \varphi \in \Gamma \qquad \text{implies} \qquad Q = P\,.
\end{equation*}
\end{definition}

\begin{definition}[Admissible set]
Let $\Gamma \subset C_b(\mathcal{X})$ be equipped with the subspace topology inherited from the weak topology. We say that $\Gamma$ is \emph{admissible} if it satisfies the following conditions:
\begin{enumerate}
\item $\Gamma$ is convex and closed;
\item $\Gamma$ is symmetric, i.e., $g \in \Gamma$ implies $-g \in \Gamma$, and it contains all constant functions;
\item $\Gamma$ is determining for $\mathcal{P}(\mathcal{X})$, i.e., for any
$\mu,\nu \in \mathcal{P}(\mathcal{X})$ with $\mu \neq \nu$, there exists
$g \in \Gamma$ such that
\begin{equation*}
  \int_{\mathcal{X}} g \, d\mu \neq \int_{\mathcal{X}} g \, d\nu \,.
\end{equation*}
\end{enumerate}
\end{definition}

\begin{definition}[Strictly admissible set] $\Gamma$ will be called \emph{strictly admissible} if it also satisfies the following property: there exists a $\mathcal{P}(\mathcal X)$-determining set $\Psi \subset C_b(\mathcal X)$ such that for all $\varphi \in \Psi$ there exist $c \in \mathbb{R}$ and $\varepsilon > 0$ such that $c \pm \varepsilon \varphi \in \Gamma$.
\end{definition}

\section{Supporting material}\label{supplem}

In this section, we prove the auxiliary statements used in Section~\ref{subsec_examples_adm_sets} for the Lipschitz discriminator setting. We present the argument first for the simple graph structure in \eqref{Structure1}, since the argument extends directly to any DAG
$G$ without essential modification.

The next calculation explains why a Lipschitz dependence of the generator on the conditioning variables leads naturally to a Wasserstein-Lipschitz family of conditional distributions. Let $(\mathcal Z,\mathcal F,\nu)$ be a probability space, $\mathcal X_{\mathrm{Pa}(i)}\subset\mathbb R^m$ and $g:\mathcal X_{\mathrm{Pa}(i)}\times \mathcal Z\to\mathbb R^d$. We assume that 
for every $\bm x_{\mathrm{Pa}(i)}\in\mathcal X_{\mathrm{Pa}(i)}$, the map $z\mapsto g(\bm x_{\mathrm{Pa}(i)},z)$ is measurable 
and  $g(\bm x_{\mathrm{Pa}(i)},\cdot)_{\#}\nu\in\mathcal P_1(\mathbb R^d)$, and for any $\bm x_{\mathrm{Pa}(i)},\bm x'_{\mathrm{Pa}(i)}\in\mathcal X_{\mathrm{Pa}(i)}$ and all $z\in\mathcal Z$, 
$\|g(\bm x_{\mathrm{Pa}(i)},z)-g(\bm x'_{\mathrm{Pa}(i)},z)\|\leq Cd_{\mathcal{X}_{\mathrm{Pa}(i)}}(\bm x_{\mathrm{Pa}(i)},\bm x'_{\mathrm{Pa}(i)})$. Then, a measurable generator with finite first moment and uniform Lipschitz dependence on the conditioning random vector induces a $W_1$-Lipschitz family of conditional distributions. 

We begin by verifying this claim in the simplest nontrivial case, namely the terminal conditional distribution in the chain \eqref{Structure1}.
\begin{equation}
\begin{split}
&W_1 \left( Q( dx_{3} |x_2) ,Q(d x_3 |x'_2)\right) \nonumber\\
    &\qquad=\sup_{f\in\mathrm{Lip}_b^1(\mathcal{X}_3)} \left\{ \int_{\mathcal X_{3}}f(x_{3})Q(d x_{3} | x_2)-\int_{\mathcal X_{3}}f( x_{3})Q(d x_{3} | x'_2) \right\} \nonumber\\
    &\qquad=\sup_{f\in\mathrm{Lip}_b^1(\mathcal{X}_3)} \left\{ \int_{\mathcal{X}_3}f( x_{3})(g(x_2,\cdot)_{\#})\,\nu(dz)-\int_{\mathcal{X}_3}f( x_{3})(g(x'_2,\cdot)_{\#})\,\nu(dz) \right\} \nonumber\\
    &\qquad=\sup_{f\in\mathrm{Lip}_b^1(\mathcal{X}_3)} \left\{ \int_{\mathcal{Z}}f(g(x_2,z))\,\nu(dz)-\int_{\mathcal{Z}}f(g( x'_2,z))\,\nu(dz) \right\} \nonumber\\
    &\qquad\leq\int_{\mathcal{Z}} \| g( x_2,z)-g( x'_2,z) \|\,d\nu(z)\nonumber\\
    &\qquad\leq Cd_{\mathcal{X}_2}( x_2, x'_2)
\end{split}
\end{equation}
The proof of $W_1\left(Q(d x_{2} |x_1) ,Q(d x_2 | x'_1)\right)\leq Cd_{\mathcal{X}_1}( x_1, x'_1) $ goes similarly.

We now formalize the preceding observation in a proposition that is used to verify condition {\bf (C)} for Lipschitz-bounded discriminator classes. For the next proposition, we denote $$K_i(d\bm x_{ [i\,\cup\,\mathrm{Pa}(i)]^c}|\bm x_{i\,\cup\,\mathrm{Pa}(i)}):=Q_{V^{(i-1)}|\mathrm{Pa}(i)}(d\bm x_{V^{(i-1)}}) \prod_{j=i+1}^n Q_{j |\mathrm{Pa}(j)}(dx_j),$$
where we remind that $V^{(i)}=\{1,\dots,i\}\setminus \mathrm{Pa}(i+1)$. The next proposition is stated for DAG \eqref{Structure1} for  simplicity, it can be easily generalized to an arbitrary DAG.
\begin{proposition}
Let $G$ be the chain $X_1 \to X_2 \to X_3$. Assume that
\[
W_1(Q_{2|1}(\cdot| x_1),Q_{2|1}(\cdot| x_1'))\le C_2\,d_{\mathcal X_1}(x_1,x_1'),
\]
\[
W_1(Q_{3|2}(\cdot|x_2),Q_{3|2}(\cdot| x_2'))\le C_3\,d_{\mathcal X_2}(x_2,x_2').
\]
Then $K_i$ is Lipschitz with Lipschitz constant $C$, with respect to the $1$-Wasserstein distance.
\end{proposition}

\begin{proof}
We prove the result for $i=1$ (the proof for $i=2,3$ follows the same steps), that is, we prove that 
\[
W_1(K_1(\cdot| x_1),K_1(\cdot| x_1'))
\le
(1+C_3)C_2\,d_{\mathcal X_1}(x_1,x_1').
\]
Fix $x_1,x_1'$. Let $\pi_{22'}$ be an optimal coupling of
$Q_{2|1}(\cdot| x_1)$ and $Q_{2|1}(\cdot| x_1')$, so that
\[
\int d_{\mathcal X_2}(x_2,x_2')\,\pi_{22'}(dx_2,dx_2')
=
W_1((Q_{2|1}(\cdot| x_1),Q_{2|1}(\cdot| x_1')).
\]

For each $(x_2,x_2')$, let $\pi_{33'}^{x_2,x_2'}$ be an optimal coupling of
$Q_{3|2}(\cdot| x_2)$ and $Q_{3|2}(\cdot| x_2')$, i.e.,
\[
\int d_{\mathcal X_3}(x_3,x_3')\,\pi_{33'}^{x_2,x_2'}(dx_3,dx_3')
=
W_1(Q_{3|2}(\cdot| x_2),Q_{3|2}(\cdot| x_2')).
\]

Define a coupling $\Pi$ of $K_1(\cdot| x_1)$ and $K_1(\cdot| x_1')$ by
\[
\Pi(dx_2,dx_3,dx_2',dx_3')
=
\pi_{22'}(dx_2,dx_2')\,\pi_{33'}^{x_2,x_2'}(dx_3,dx_3').
\]
It is straightforward to verify that $\Pi$ has the correct marginals.

By definition of $W_1$ and using the product metric,
\begin{eqnarray*}
W_1(K_1(\cdot| x_1),K_1(\cdot| x_1'))
&\le&
\int \bigl(d_{\mathcal X_2}(x_2,x_2')+d_{\mathcal X_3}(x_3,x_3')\bigr)\,d\Pi \\
&=&
\int d_{\mathcal X_2}(x_2,x_2')\,\pi_{22'}(dx_2,dx_2') \\
&&\quad+
\int\int
\left(\int d_{\mathcal X_3}(x_3,x_3')\,\pi_{33'}^{x_2,x_2'}(dx_3,dx_3')\right)
\pi_{22'}(dx_2,dx_2').
\end{eqnarray*}

Using optimality of the couplings, this yields
\begin{eqnarray*}
W_1(K_1(\cdot| x_1),K_1(\cdot| x_1')) 
&\le&
W_1(Q_{2|1}(\cdot| x_1),Q_{2|1}(\cdot| x_1')) \\
&&\quad+
\int
W_1(Q_{3|2}(\cdot| x_2),Q_{3|2}(\cdot| x_2'))\,
\pi_{22'}(dx_2,dx_2').
\end{eqnarray*}

By the Lipschitz assumption on $Q_{3|2}$,
\[
W_1(Q_{3|2}(\cdot| x_2),Q_{3|2}(\cdot| x_2'))
\le C_3\,d_{\mathcal X_2}(x_2,x_2'),
\]
so
\[
\int
W_1(Q_{3|2}(\cdot| x_2),Q_{3|2}(\cdot| x_2'))\,
\pi_{22'}(dx_2,dx_2')
\le
C_3
\int d_{\mathcal X_2}(x_2,x_2')\,\pi_{22'}(dx_2,dx_2').
\]

Therefore,
\begin{align*}
&W_1(K_1(\cdot| x_1),K_1(\cdot| x_1')) \\
&\le
(1+C_3)
W_1(Q_{2|1}(\cdot| x_1),Q_{2|1}(\cdot| x_1'))
\le
(1+C_3)C_2\,d_{\mathcal X_1}(x_1,x_1'),
\end{align*}
which proves the claim.
\end{proof}
\begin{proposition}
\label{prop:Lip_continuity for conditional Q}
    Let $Q\in\mathcal P^G$ and let $G$ be an arbitrary DAG. Assume that, for each
$i\in V$, the kernel
$K_i(\cdot| \bm x_{i\cup\mathrm{Pa}(i)})$
is $C$-Lipschitz with respect to $W_1$, i.e.,
\begin{equation}\label{ass:Lip1}
W_1(
K_i(\cdot| \bm x_{i\cup\mathrm{Pa}(i)}),
K_i(\cdot| \bm x'_{i\cup\mathrm{Pa}(i)})
)
\le
C\,d_{\mathcal X_{i\cup\mathrm{Pa}(i)}}(
\bm x_{i\cup\mathrm{Pa}(i)},\bm x'_{i\cup\mathrm{Pa}(i)}).
\end{equation}
If $\Gamma^L=\mathrm{Lip}_b^L(\mathcal X)$, then condition {\bf (C)} holds with
\[
L_i=L(1+C)>L,\qquad i\in V.
\]
\end{proposition}

\begin{proof}
Fix $i\in V$ and let
\[
u=\bm x_{i\cup\mathrm{Pa}(i)},\qquad
u'=\bm x'_{i\cup\mathrm{Pa}(i)}.
\]
For $\gamma\in \Gamma^L$, define
\[
I_{i,\mathrm{Pa}(i)}[\gamma](u)
=
\int_{\mathcal X_{[i\cup\mathrm{Pa}(i)]^c}}
\gamma(u,z)\,K_i(dz| u).
\]
Let $\Pi$ be any coupling of $K_i(\cdot| u)$ and $K_i(\cdot| u')$. Then
\begin{eqnarray*}
|I_{i,\mathrm{Pa}(i)}[\gamma](u)-I_{i,\mathrm{Pa}(i)}[\gamma](u')| 
&=&
\left|
\int \gamma(u,z)\,K_i(dz| u)
-
\int \gamma(u',z')\,K_i(dz'| u')
\right| \\
&=&
\left|
\int (\gamma(u,z)-\gamma(u',z'))\,\Pi(dz,dz')
\right| \\
&\le&
\int |\gamma(u,z)-\gamma(u',z')|\,\Pi(dz,dz').
\end{eqnarray*}
Since $\gamma$ is $L$-Lipschitz,
\[
|\gamma(u,z)-\gamma(u',z')|
\le
L\Big(
d_{\mathcal X_{i\cup\mathrm{Pa}(i)}}(u,u')
+
d_{\mathcal X_{[i\cup\mathrm{Pa}(i)]^c}}(z,z')
\Big).
\]
Therefore,
\begin{align*}
|I_{i,\mathrm{Pa}(i)}[\gamma](u)-I_{i,\mathrm{Pa}(i)}[\gamma](u')|
&\le
L\,d_{\mathcal X_{i\cup\mathrm{Pa}(i)}}(u,u') \\
&\quad+
L\int d_{\mathcal X_{[i\cup\mathrm{Pa}(i)]^c}}(z,z')\,\Pi(dz,dz').
\end{align*}
Choosing $\Pi$ optimally and using \eqref{ass:Lip1}, we obtain
\begin{align*}
|I_{i,\mathrm{Pa}(i)}[\gamma](u)-I_{i,\mathrm{Pa}(i)}[\gamma](u')|
&\le
L\,d_{\mathcal X_{i\cup\mathrm{Pa}(i)}}(u,u')
+
L\,W_1\!\left(K_i(\cdot| u),K_i(\cdot| u')\right) \\
&\le
L(1+C)\,d_{\mathcal X_{i\cup\mathrm{Pa}(i)}}(u,u').
\end{align*}
Boundedness is straightforward. Hence
\[
I_{i,\mathrm{Pa}(i)}[\gamma]\in \mathrm{Lip}_b^{L(1+C)}(\mathcal X_{i\cup\mathrm{Pa}(i)}),
\]
which proves the desired result. 
\end{proof}

\end{document}